\documentclass{siamart190516}
\usepackage{geometry}

\usepackage{tablefootnote}
\usepackage{graphicx}
\usepackage{subcaption}
\usepackage{dsfont}
\usepackage{amssymb}
\usepackage{enumerate}
\usepackage{graphicx}
\usepackage{float}
\usepackage{bbm}
\usepackage{amsmath}
\usepackage{comment}
\usepackage{hyperref}
\usepackage{listings}
\usepackage{color}
\usepackage{algorithmic}
\usepackage{enumitem}
\usepackage{xcolor}

\newtheorem{example}{Example}[section]

\newsiamremark{remark}{Remark}
\newsiamremark{experiment}{Experiment}

\usepackage{todonotes}

\newcommand{\rank}{{\rm rank\,}}

\def\fro{\mathrm{F}}
\def\rank{\mathsf{rank}}
\graphicspath{{figures/}}
\newcommand{\rev}[1]{\textcolor{black}{\noindent  #1}}

\title{Coseparable Nonnegative Tensor Factorization with t-CUR Decomposition \thanks{Received by the editors XXXX, 2023; accepted for publication (in revised form) XXXX XX, 2023; published electronically XXXX XX, 2023. \funding{J. Chen and Y. Wei are supported by the National Natural Science Foundation of China under grant12271108 and the Ministry of Science and Technology of China under grant G2023132005L and the Science and Technology Commission of Shanghai Municipality under grant 23JC1400501.}}}
\headers{Coseparable Nonnegative Tensor Factorization with t-CUR Decomposition}{Juefei Chen, Longxiu Huang and Yimin Wei}

\author{Juefei Chen\thanks{School of Mathematical Sciences, Fudan University.
Shanghai, 200433, P R. of China}
\and Longxiu Huang\thanks{Corresponding author. Department of Computational Mathematics, Science and Engineering and Department of Mathematics, Michigan State University, East Lansing, MI 48823, USA. (email:  \href{mailto:huangl3@msu.edu}{huangl3@msu.edu})}
\and Yimin Wei\thanks{School of Mathematical Sciences and Key Laboratory of Mathematics for Nonlinear Sciences, Fudan University.
Shanghai, 200433, P R. of China}
}

\begin{document}


\maketitle


\begin{abstract}
Nonnegative Matrix Factorization (NMF) is an important unsupervised learning method to extract meaningful features from data. To address the NMF problem within a polynomial time framework, researchers have introduced a separability assumption, which has recently evolved into the concept of coseparability. This advancement offers a more efficient core representation for the original data. However, in the real world, data is more naturally represented as a multi-dimensional array, such as images or videos. The NMF's application to high-dimensional data involves vectorization, which risks losing essential multi-dimensional correlations. To retain these inherent correlations, we turn to tensors (multi-dimensional arrays) and leverage the tensor t-product. This approach extends the coseparable NMF to the tensor setting, creating what we term coseparable Nonnegative Tensor Factorization (NTF). In this work, we provide an alternating index selection method to select the coseparable core. Furthermore, we validate the t-CUR sampling theory and integrate it with the tensor Discrete Empirical Interpolation Method (t-DEIM) to introduce an alternative, randomized index selection process. These methods have been tested on both synthetic and facial analysis datasets.   The results demonstrate the efficiency of coseparable NTF when compared to coseparable NMF.
\end{abstract}



\section{Introduction}
In data science, unsupervised learning plays an important role in dealing with unlabeled data, as it can uncover unseen patterns, features, and structures within the data. Nonnegative Matrix Factorization (NMF) is a significant unsupervised learning method to extract meaningful features from data \cite{buciu2008non}. Given a nonnegative matrix $A\in\mathbb{R}_{\geq 0}^{m\times n}$ and  a target rank $r<\min\{m,n\}$, NMF approximates $A$ by the product of two non-negative low-rank matrices: the \textit{dictionary matrix} $W \in \mathbb{R}_+^{m \times r}$ and the \textit{coding matrix} $H \in \mathbb{R}_+^{r \times n}$ i.e., $A\approx WH$. 
An important application of NMF is topic modeling, which can extract and classify topics from the given word-document data \cite{CBCL2,kuang2015nonnegative,xu2003document}. It can be used for tasks like text mining, sentiment analysis, and news clustering. To solve the NMF problem in polynomial time, a separability assumption is proposed in \cite{separability}, i.e., $W$ is composed of some columns of $A$. Then some researchers proposed many algorithms to solve the separable NMF problem \cite{separableNMF1, separableNMF2, separableNMF3}. Recently, Pan and Ng generalized the separability assumption into coseparability in \cite{pan2021co}, which assumes that $W$ is composed of some rows and columns of $A$. In other words, $W$ is a submatrix of $A$. It provides a more compact core matrix to represent the original data matrix $A$.

For high-dimensional data like images or videos, using NMF for clustering or feature extraction necessitates vectorization. However, this vectorization process can potentially disrupt inherent correlations in higher dimensions. For instance, converting an image into a vector might result in the loss of relational context between adjacent pixels. So we aim to find a similar method that can process high-dimensional data while preserving their structure. In recent years, tensors have been widely studied as a structure for high-dimensional data, and different techniques have been studied. The CANDECOMP/PARAFAC (CP) and Tucker decompositions \cite{tucker1, tucker2, tucker3} can be considered higher-order generalizations of the matrix singular value decomposition (SVD). Oseledets developed tensor train decomposition \cite{train} and Zhao et al. introduced tensor ring decomposition \cite{ring}. Another type of tensor factorization which is useful in applications is the tensor t-product, founded by Kilmer and Martin in \cite{Tproduct}. Based on the t-product, the matrix factorizations can be extended to tensors, such as t-SVD \cite{Tproduct}, t-QR/t-LU decomposition \cite{TQR}, and t-Schur decomposition \cite{TSchur1, TSchur2}. Han et. al. recently proposed an adaptive data augmentation framework based on the t-product in \cite{TADAF}. The matrix CUR decomposition, which  utilizes self-expression to reduce the size of the initial matrix, is widely investigated in \cite{CUR1, CUR2, TCURproof}, and it was extended to t-CUR decompositions in \cite{TCUR}. 

In the tensor t-product framework, NMF has also been extended to a high-dimensional setting called nonnegative tensor factorization (NTF)  \cite{NTF1}. Further studies in NTF can be found in \cite{NTF2}.  The recent work of coseparable NMF by Pan and Ng is a generalization of separability, and it has a strong relation to CUR decomposition, meaning a CUR-sampling-like method can be used to randomly select the coseparable core. Consequently, in this paper, we choose the coseparable NMF for extension and establish the coseparable NTF theory with the tensor t-product. Unlike most analyses of the tensor t-product being transformed into the Fourier domain, in the case of NTF analysis, the nonnegativity and other NMF properties may not be preserved in the Fourier domain. So we use the original definition of t-product to derive our coseparable NTF properties and algorithms in this paper.

 Our coseparable NTF exhibits  interpretability akin to coseparable NMF but preserves information in higher dimensions. For instance, in clustering tasks on facial datasets, the coseparable NMF identifies the ``significant'' images for each subject along with their key pixels. However, the coseparable NTF core extracts key pixel vectors, such as those representing the eyes, nose, and mouth, retaining high-dimensional features as a whole.

The contribution of this paper can be summarized as follows: 
\begin{enumerate}
    \item We extend the coseparable NMF to tensors and propose coseparable NTF. Some of its properties have also been proven, including its relationship with t-CUR decomposition. \rev{This coseparable NTF offers a framework for handling higher-dimensional data, one can easily further generalize it by setting the matrix blocks as tensors and employing a similar t-product methodology.} To solve our coseparable NTF problem,  an alternating index selection algorithm is proposed to choose the coseparable core. 
    \item Inspired by matrix CUR sampling theory  in \cite{CURsample}, we present the t-CUR sampling theory, which shows that randomly sampling indices according to different probability distributions can achieve the t-CUR decomposition with high probability.  Combining it with the tensor Discrete Empirical Interpolation Method (t-DEIM) \cite{TDEIM}, an alternative t-CUR-DEIM method is proposed to select coseparable cores. 
    \item We test two methods on synthetic coseparable tensor data sets. We also test them on several real facial data sets and compare them with some matrix index selection methods.
\end{enumerate}

The structure of this paper is as follows: \Cref{sec:pre} introduces the concept of tensor t-product. Our theoretical contributions regarding Coseparable Non-negative Tensor Factorization (CoS-NTF) are detailed in \Cref{sec:CoNTF}. Additionally, \Cref{sec:index-selection} presents the index selection algorithm for identifying the subtensor of the original tensor. Numerical validations of our theoretical results are provided in \Cref{sec:simulation}. The paper concludes with \Cref{sec:conclusion}.

\section{Preliminaries} \label{sec:pre}

\subsection{Notation}
In this paper, we adopt the following notation for clarity and consistency. Matrices are denoted by uppercase italic letters (e.g., \(A\)), while third-order tensors are represented by uppercase cursive italics (e.g., \(\mathcal{A}\)). The space of nonnegative matrices and third-order tensors is symbolized as \(\mathbb{R}^{m\times n\times p}_+\) (with \(\mathbb{R}^{m\times n}_+=\mathbb{R}^{m\times n\times 1}_+\) for matrices). We use \(\{a:a+n\}\) to represent the integer set \(\{a, a+1, \cdots, a+n\}\). \(\mathcal{A}_{ijk}\) denotes the \((i,j,k)\)-th entry of tensor \(\mathcal{A}\) and \(\mathcal{A}_{\mathbf{IJK}}\) denotes the subtensor of \(\mathcal{A}\) whose entries $\mathcal{A}_{ijk}$ satisfy $i\in\mathbf{I}$, $j\in\mathbf{J}$, $k\in\mathbf{K}$ for sets $\mathbf{I,J,K}$. Specifically, \(\mathcal{A}_{:jk}\) refers to \(\mathcal{A}_{1:m,j,k}\). The cardinality of a set \(\mathbf{s}\) is expressed as \(|\mathbf{s}|\).

For matrix operations, \(A^{\top}\) and \(A^{\ast}\) denote the transpose and conjugate transpose of \(A\), respectively. The inverse and Moore-Penrose pseudoinverse of \(A\) are represented by \(A^{-1}\) and \(A^{\dag}\). The Kronecker product of matrices \(A\) and \(B\) is indicated as \(A\otimes B\). Norms are specified as \(\|\cdot\|_2\) for the matrix spectral norm and \(\|\cdot\|_{\fro}\) for the Frobenius norm. \(I_n\) denotes the \(n \times n\) identity matrix, while \(F_n\) signifies the \(n \times n\) discrete Fourier transform (DFT) matrix, defined as:
\[
F_n=\frac{1}{\sqrt{n}}\left( f_{jk} \right),\; f_{jk}=\omega _{n}^{jk},\; \omega _n=\mathrm{e}^{-\frac{2\pi \iota}{n}}, \; \iota=\sqrt{-1}.
\]

We also use functions in MATLAB to denote some operations:  \(\widehat{\mathcal{A}}:= \mathsf{fft}( \mathcal{A},[],3 ) \) and \(\mathcal{A}:= \mathsf{ifft}(\widehat{\mathcal{A}},[],3 ) \) for the Fast Fourier Transform (FFT) and the inverse FFT along the third dimension of tensors \(\mathcal{A}\) and \(\widehat{\mathcal{A}}\), respectively. \(\mathsf{diag}( \mathbf{v} )\) denotes the diagonal matrix formed by array \(\mathbf{v}=( v_1,v_2,\cdots ,v_n )\).

If $\mathsf{rank}\left(A\right) = r$, then the SVD  of $A\in\mathbb{R}^{m\times n}$ can be represented as
$$A=W\Sigma V^{\top},$$
where $\Sigma=\mathsf{diag}\left( \sigma _1\left( A \right) ,\sigma _2\left( A \right) ,\cdots ,\sigma _r\left( A \right),0,\cdots,0 \right)$, $\sigma _1\left( A \right) \geqslant \sigma _2\left( A \right) \geqslant \ldots \geqslant  \sigma _r\left( A \right) >0$, with $\sigma _i\left(A \right)$ being the $i$-th singular value, $W$ and $V$ are   the left and right singular vector matrices of $A$. We denote $\sigma_{\min}\left(A\right)$ as the smallest nonzero singular value for convenience. 

In the context of this paper, which extends the concept of coseparable NMF to tensors, we define coseparable NMF for reference as follows.
\begin{definition}\label{t1.1}
A matrix ${A}\in\mathbb{R}^{m\times n}_+$ is co-$(r_1,r_2)$ separable if there exists index sets $\mathbf{I}$, $\mathbf{J}$ and matrices ${P}_1\in\mathbb{R}^{m\times r_1}_+$, ${P}_2\in\mathbb{R}^{r_2\times n}_+$ such that ${A} = {P}_1{A}_{\mathbf{I}\mathbf{J}}{P}_2$, where $|\mathbf{I}|=r_1$, $|\mathbf{J}|=r_2$, $\left({P}_1\right)_{\mathbf{I}:}={I}_{r_1}$ and $\left({P}_2\right)_{:\mathbf{J}}={I}_{r_2}$. ${A}_{\mathbf{I}\mathbf{J}}$ is referred to as the core of matrix ${A}$.
\end{definition}

\subsection{Tensor t-product}
A tensor can be regarded as a high-dimensional array. As shown in   \Cref{f1}, the left-hand side is an $m\times n\times p$ third-order tensor. Figure \ref{f1} (a), (b), and (c) are called row, column, and tube fibers respectively; (d), (e), and (f) are called horizontal, lateral, and frontal slices respectively.
\begin{figure}[ht]
\begin{center}
\includegraphics[width=12cm]{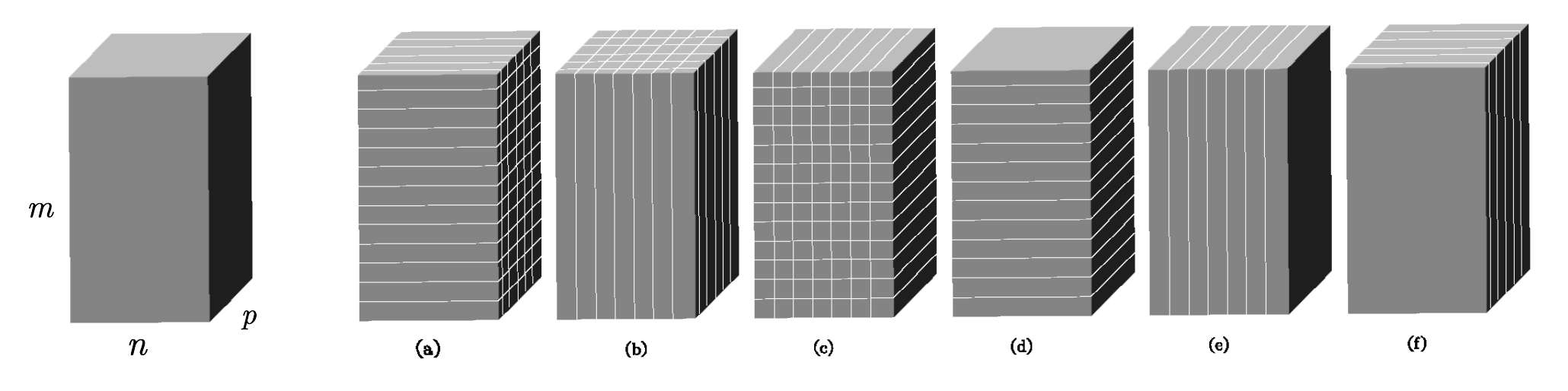}
\caption{Illustration of a third-order tensor}
\label{f1}
\end{center}
\end{figure}
The addition and subtraction of two $m\times n\times p$ tensors $\mathcal{A}$ and $\mathcal{B}$ is defined as
\begin{equation}\label{e2.1}
\left(\mathcal{A}\pm \mathcal{B}\right)_{ijk}:= \mathcal{A}_{ijk}\pm \mathcal{B}_{ijk},\quad  i\in \{ 1:m \},\ j\in \{1:n \},\ k\in \{1:p \}.
\end{equation}

In preparation for defining the tensor t-product, we need operations $\mathsf{bcirc}$, $\mathsf{unfold}$ and $\mathsf{fold}$.
\begin{definition}{\rm (\cite{Tproduct})}\label{t2.1}
Let $\mathcal{A}\in \mathbb{R}^{m\times n\times p}$. Then
\begin{equation}\label{e2.2}
{\mathsf{bcirc}}(\mathcal{A}):=\left[ \begin{matrix}
	\mathcal{A}_{::1}&		\mathcal{A}_{::p}&		\cdots&		\mathcal{A}_{::2}\\
	\mathcal{A}_{::2}&		\mathcal{A}_{::1}&		\cdots&		\mathcal{A}_{::3}\\
	\vdots&		\vdots&		\ddots&		\vdots\\
	\mathcal{A}_{::p}&		\mathcal{A}_{::,p-1}&		\cdots&		\mathcal{A}_{::1}\\
\end{matrix} \right],\quad
{\mathsf{unfold}}(\mathcal{A}):=\left[ \begin{array}{c}
	\mathcal{A}_{::1}\\
	\mathcal{A}_{::2}\\
	\vdots\\
	\mathcal{A}_{::p}\\
\end{array} \right] , \text{ and }
{\mathsf{fold}( \mathsf{unfold}}(\mathcal{A})) := \mathcal{A},
\end{equation}
where ${\mathsf{bcirc}}(\mathcal{A})$ is called the block circulant matrix generated by $\mathcal{A}$.
\end{definition}

\begin{definition}{\rm (\cite{Tproduct}, t-product)}\label{t2.2}
The t-product of $\mathcal{A}\in \mathbb{R}^{m\times n\times p}$ and $\mathcal{B}\in \mathbb{R}^{n\times q\times p}$ is denoted as $\mathcal{A}\ast \mathcal{B}\in \mathbb{R}^{m\times q\times p}$ and 
\begin{equation}\label{e2.3}
\mathcal{A}\ast \mathcal{B}:= \mathsf{fold} \left(\mathsf{bcirc} \left( \mathcal{A} \right) \mathsf{unfold}\left( \mathcal{B} \right)\right).
\end{equation}
\end{definition}

The $\mathsf{bcirc}\left( \mathcal{A} \right)$ can be transformed into the Fourier domain as a block diagonal matrix \cite{Tproduct}
\begin{align}\label{e2.4}
\left( F_{p}\otimes I_{m} \right) \mathsf{bcirc}\left( \mathcal{A} \right) \left( F_{p}^{\ast}\otimes I_{n} \right) :=\widehat{A}=\left[ \begin{matrix}
	\widehat{A}_1&		&		&		\\
	&		\widehat{A}_2&		&		\\
	&		&		\ddots&		\\
	&		&		&		\widehat{A}_{p}\\
\end{matrix} \right].
\end{align}
Note that \eqref{e2.3} is equivalent to
$\mathsf{bcirc} \left(\mathcal{A}\ast \mathcal{B}\right)= \mathsf{bcirc} \left( \mathcal{A} \right) \mathsf{bcirc}\left( \mathcal{B} \right)$, then
\begin{align}\label{e2.5}
\left( F_p\otimes I_m \right) \mathsf{bcirc}\left( \mathcal{A} \ast \mathcal{B} \right) \left( F_{p}^{\ast}\otimes I_q \right) = \left( F_p\otimes I_m \right) \mathsf{bcirc} \left( \mathcal{A} \right)\mathsf{bcirc}\left( \mathcal{B} \right) \left( F_{p}^{\ast}\otimes I_q \right) =\widehat{A}\widehat{B},
\end{align}
which means the t-product can be computed in the Fourier domain. We have summarized the details of the tensor t-product in \Cref{a1}. 

More precisely, we can fold the diagonal blocks $\widehat{A}_k$ in \eqref{e2.4} into a tensor $\widehat{\mathcal{A}}$ such that $\widehat{\mathcal{A}}_{::k}=\widehat{A}_k$, which is equivalent to applying the FFT along the third dimension of $\mathcal{A}$, i.e., $\widehat{\mathcal{A}}=\mathsf{fft}(\mathcal{A},[],3)$. The same procedure is applied to $\mathcal{B}$, then we can perform matrix multiplication for their corresponding frontal slices in the complex domain. Consequently, the t-product can be efficiently computed using the FFT, as shown in Algorithm \ref{a1}.
\begin{algorithm}
\caption{Tensor t-product}
\label{a1}
\begin{algorithmic}
\STATE $\mathbf{input}\ \mathcal{A}\in \mathbb{R}^{m\times n\times p}$ and $\mathcal{B}\in \mathbb{R}^{n\times q\times p}$
\STATE $\widehat{\mathcal{A}}= \mathsf{fft}\left( \mathcal{A},\left[\right] ,3 \right) $
\STATE $\widehat{\mathcal{B}}= \mathsf{fft}\left( \mathcal{B},\left[\right] ,3 \right) $
\FOR{$k=1:p$}
\STATE $\widehat{\mathcal{C}}\left(:,:,k\right)=\widehat{\mathcal{A}}\left(:,:,k\right) \widehat{\mathcal{B}}\left(:,:,k\right)$
\ENDFOR
\STATE ${\mathcal{C}}= \mathsf{ifft}\left( \widehat{\mathcal{C}},\left[\right] ,3 \right) $
\RETURN $\mathcal{C}=\mathcal{A}\ast\mathcal{B}$
\end{algorithmic}
\end{algorithm}

Now we review some definitions and technical results in the tensor t-product framework that will be used in the following sections.
\begin{definition}{\rm (\cite{Tproduct}, Identity tensor)}\label{t2.3}
A tensor $\mathcal{I}\in \mathbb{R}^{n\times n\times p}$ is the identity tensor, if $\mathcal{I}_{::1}=I_n$ and $\mathcal{I}_{::k}=O\in\mathbb{R}^{n\times n}$, the zero matrix for $k\in\{2:p\}$.
\end{definition}

\begin{definition}{\rm (\cite{Tproduct}, Transpose and orthogonal tensor)} \label{t2.4}
\\
$\mathrm{(a)}$ The transpose of $\mathcal{A}\in \mathbb{R}^{m\times n\times p}$ is $\mathcal{A}^{\top}\in \mathbb{R}^{n\times m\times p}$, where $\left(\mathcal{A}^{\top}\right)_{::1} =\left(\mathcal{A}_{::1}\right) ^{\top}$, $\left(\mathcal{A}^{\top}\right)_{::k} =\left(\mathcal{A}_{::p-k+2}\right) ^{\top}$, $k \in \{2:p\}$.
\\
$\mathrm{(b)}$ A tensor $\mathcal{A}\in \mathbb{R}^{m\times n\times p}$ is orthogonal, if $\mathcal{A}^{\top}\ast \mathcal{A}=\mathcal{A}\ast \mathcal{A}^{\top}=\mathcal{I}$, where $\mathcal{I}$ is the identity tensor, see \Cref{t2.3}.
\end{definition}

\begin{definition}{\rm (\cite{Tproduct, Tpinv}, Inverse and Moore-Penrose inverse)}\label{t2.5} 
\\
$\mathrm{(a)}$ A tensor $\mathcal{A}^{-1}\in \mathbb{R}^{n\times n\times p}$ is the inverse of $\mathcal{A}\in \mathbb{R}^{n\times n\times p}$, if $\mathcal{A}\ast\mathcal{A}^{-1}=\mathcal{A}^{-1}\ast\mathcal{A}=\mathcal{I}$, the identity tensor.
\\
$\mathrm{(b)}$ A tensor $\mathcal{A}^{\dag}\in \mathbb{R}^{n\times m\times p}$ is the Moore-Penrose inverse of $\mathcal{A}\in \mathbb{R}^{m\times n\times p}$, if $\mathcal{A}^{\dag}$ satisfies the following equations,
$$\mathcal{A}\ast \mathcal{A}^{\dag}\ast \mathcal{A}=\mathcal{A},\quad
\mathcal{A}^{\dag}\ast \mathcal{A}\ast \mathcal{A}^{\dag}=\mathcal{A}^{\dag},\quad
\left( \mathcal{A}\ast \mathcal{A}^{\dag} \right) ^{\top}=\mathcal{A}\ast \mathcal{A}^{\dag}, \quad
\left( \mathcal{A}^{\dag}\ast \mathcal{A} \right) ^{\top}=\mathcal{A}^{\dag}\ast \mathcal{A}.$$
\end{definition}

\begin{definition}{\rm (\cite{Tproduct, TCURproof}, Tensor norm)}\label{t2.6} 
The Frobenius norm and the spectral norm of tensor $\mathcal{A}\in\mathbb{R}^{m\times n\times p}$ are defined as
\begin{align}\label{e2.6}
\lVert \mathcal{A} \rVert _{\fro}:= \sqrt{\sum_k{\lVert \mathcal{A}_{::k} \rVert^2 _{\fro}}}=\sqrt{\sum_{i,j,k}{\left| \mathcal{A}_{ijk} \right|^2}},\quad
\lVert \mathcal{A} \rVert _2:= \lVert \widehat{A} \rVert _2.
\end{align}
\end{definition}

\begin{remark}\label{t2.7}
Definition \ref{t2.2} indicates that $\left\| \mathsf{bcirc}\left( \mathcal{A} \right) \right\| _{\fro}^{2}=p\left\| \mathcal{A} \right\| _{\fro}^{2}$.
Since the Frobenius norm is unitarily invariant, we can use \eqref{e2.4} to derive
\begin{align}\label{e7}
p\left\| \mathcal{A} \right\| _{\fro}^{2}=\left\| \left( F_p\otimes I_m \right) \mathsf{bcirc}\left( \mathcal{A} \right) \left( F_{p}^{\ast}\otimes I_n \right) \right\| _{\fro}^{2}=\left\| \widehat{A} \right\| _{\fro}^{2}=\sum_{k=1}^p{\left\| \widehat{A}_k \right\| _{\fro}^{2}}=\sum_{k=1}^p{\left\| \widehat{\mathcal{A}} _{::k} \right\| _{\fro}^{2}}=\left\| \widehat{\mathcal{A} } \right\| _{\fro}^{2}.
\end{align}
\end{remark}

We call a tensor f-diagonal if each of its frontal slices is a diagonal matrix \cite{Tproduct}. In the t-product framework, the tensor still has its SVD, called t-SVD, which can factorize a tensor into two orthogonal tensors and an f-diagonal tensor.

\begin{theorem}{\rm (\cite{Tproduct}, t-SVD)}\label{t2.8}
A tensor $\mathcal{A}\in \mathbb{R}^{m\times n\times p}$ can be decomposed by
\begin{equation}\label{e2.8}
\mathcal{A}=\mathcal{W}\ast \varSigma \ast \mathcal{V}^{\top},
\end{equation}
where $\mathcal{W}\in \mathbb{R}^{m\times m\times p}$ and $\mathcal{V}\in \mathbb{R}^{n\times n\times p}$ are orthogonal and $\varSigma$ is f-diagonal.
\end{theorem}

Now we will introduce the tensor multirank and the tensor tubalrank, where the multirank guarantees the t-CUR decomposition, and the tubalrank is useful in the t-CUR sampling theory in the following sections.

\begin{definition}{\rm (\cite{tensorrank}, Multirank and tubalrank)}\label{t2.9}
\\
$\mathrm{(a)}$ Let the tensor $\mathcal{A}\in \mathbb{R}^{m\times n\times p}$ satisfy $ \mathsf{rank}\left( \widehat{A}_k \right) =r_k $, $k\in \{1:p\}$, then the multirank of $\mathcal{A}$, denoted as $\mathsf{rank}_m\left( \mathcal{A} \right)$, is defined as the vector $\mathbf{r}=\left( r_1,r_2,\cdots ,r_p \right) $.
\\
$\mathrm{(b)}$ Let $\mathcal{A}=\mathcal{W}\ast \varSigma \ast \mathcal{V}^{\top}$ be the t-SVD of $\mathcal{A}\in \mathbb{R}^{m\times n\times p}$, then the tubalrank of $\mathcal{A}$, denoted as $\mathsf{rank}_t\left( \mathcal{A} \right)$, is defined as
\begin{align*}
\mathsf{rank}_t\left( \mathcal{A} \right)= \left| \left\{ i: \varSigma_{ii:} \ne 0 \right\}\right|,
\end{align*}
the number of nonzero singular values tubes of $\varSigma$. 
\end{definition}

\begin{theorem}{\rm (\cite{TCUR}, t-CUR decomposition)}\label{t2.10}
Let the tensor $\mathcal{A}\in \mathbb{R}^{m\times n\times p}$ have $\mathsf{rank}_m\left( \mathcal{A} \right)=\mathbf{r}$ and $\mathsf{rank}_t\left( \mathcal{A} \right)={r}$. Index sets $\mathbf{I}\subset \{1:m\}$ and $\mathbf{J}\subset \{1:n\} $ satisfy $\left| \mathbf{I} \right|, \left| \mathbf{J} \right|\geqslant r$. $\mathcal{C}=\mathcal{A}_{:\mathbf{J}:} $, $\mathcal{R}=\mathcal{A}_{\mathbf{I}::} $ and $\mathcal{U}=\mathcal{A}_{\mathbf{I}\mathbf{J}:} $. If $\mathsf{rank}_m\left( \mathcal{U} \right) =\mathsf{rank}_m\left( \mathcal{A} \right) =\mathbf{r}$, then
\begin{equation}\label{e2.9}
\mathcal{A}=\mathcal{C}\ast \mathcal{U}^{\dag}\ast \mathcal{R}.
\end{equation}
\end{theorem}

\begin{figure}[ht]
\begin{center}
\includegraphics[width=8.7cm]{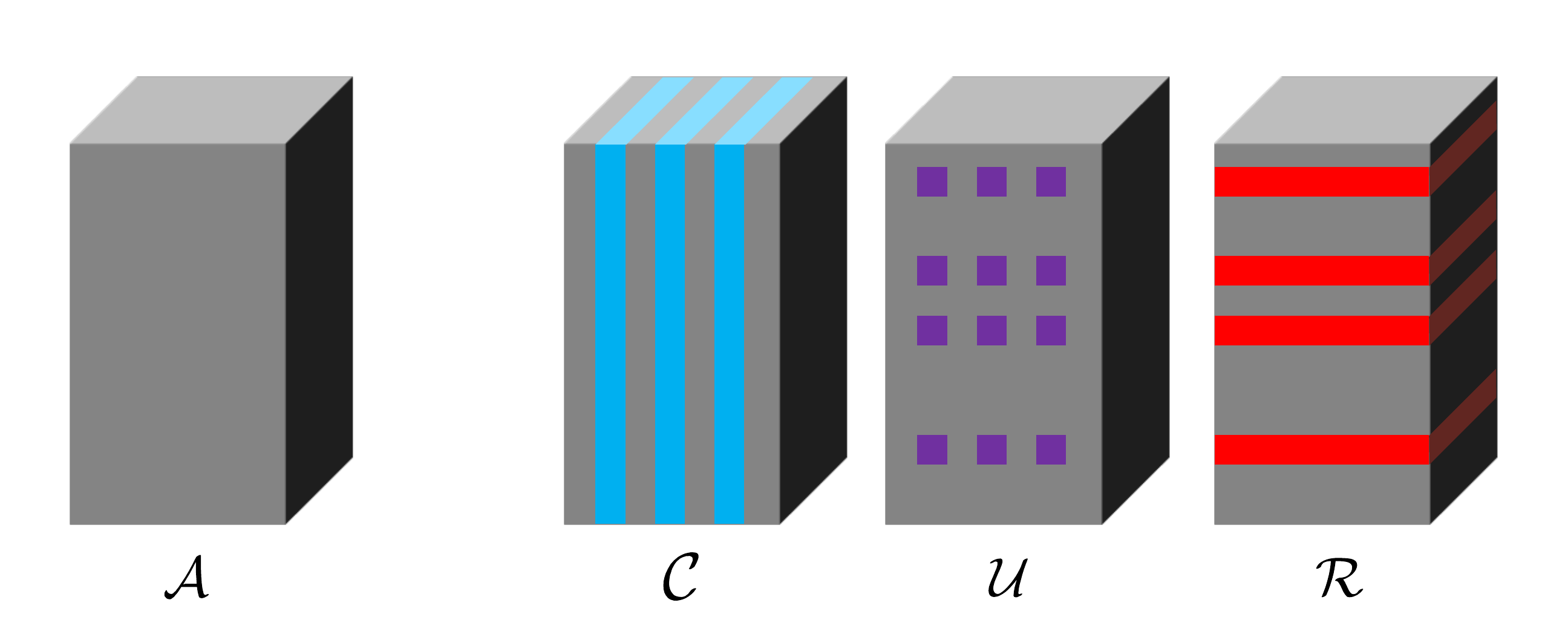}
\caption{Illustration of t-CUR decomposition.  The horizontal subtensor $\mathcal{R}$ is highlighted in red,  the lateral subtensor $\mathcal{C}$ in blue, and their intersection $\mathcal{U}$ in purple. The original tensor $\mathcal{A}\approx\mathcal{C}*\mathcal{U}^{\dagger}*\mathcal{R}$, and the approximation is exact when $\rank_{m}(\mathcal{U})=\rank_m(\mathcal{T})$.}
\label{f2}
\end{center}
\end{figure}

\section{Coseparable NTF} \label{sec:CoNTF}
In this section, we will discuss coseparability of tensors and explore their characterizations and properties.
\begin{definition}{\rm (Coseparability)}\label{t3.1}
A tensor $\mathcal{A}\in\mathbb{R}^{m\times n\times p}_+$ is co-$(r_1,r_2)$ separable if there exists index sets $\mathbf{I}$, $\mathbf{J}$ and tensors $\mathcal{P}_1\in\mathbb{R}^{m\times r_1\times p}_+$, $\mathcal{P}_2\in\mathbb{R}^{r_2\times n\times p}_+$ such that
\begin{equation}\label{e3.1}
\mathcal{A} = \mathcal{P}_1\ast \mathcal{A}_{\mathbf{I}\mathbf{J}:}\ast\mathcal{P}_2,
\end{equation}
where $|\mathbf{I}|=r_1$, $|\mathbf{J}|=r_2$, $\left(\mathcal{P}_1\right)_{\mathbf{I}::}=\mathcal{I}_{r_1}$ and $\left(\mathcal{P}_2\right)_{:\mathbf{J}:}=\mathcal{I}_{r_2}$. $\mathcal{A}_{\mathbf{I}\mathbf{J}:}$ is referred to as the core of tensor $\mathcal{A}$.
\end{definition}

 Notice that the tensor in \cref{t3.1} is nonnegative, so  \cref{e2.4} can only be used to compute t-product as \cref{e2.5}, while it may not retain the coseparable NMF properties. Therefore, we utilize \cref{t2.1} to derive our coseparable NTF properties and algorithms in the following parts. 

\begin{proposition}{\rm (Equivalent characterization)}\label{t3.2}
Let tensor $\mathcal{A}\in\mathbb{R}^{m\times n\times p}_+$, and permutation tensors $\varPi _1\in\mathbb{R}^{n\times n\times p}$, $\varPi _2\in\mathbb{R}^{m\times m\times p}$, where except for $\left(\varPi _1\right)_{::1}\in\{0,1\}^{n\times n}$, $\left(\varPi _2\right)_{::1}\in\{0,1\}^{m\times m}$, the other entries of $\varPi _1$ and $\varPi _2$ are all zeros. And let tensors $\mathcal{S}\in\mathbb{R}^{r_1\times r_2\times p}_+$, $\mathcal{M}\in\mathbb{R}^{\left(m-r_1\right)\times r_1\times p}_+$, $\mathcal{H}\in\mathbb{R}^{r_2\times \left(n-r_2\right)\times p}_+$,
\begin{align}\label{e3.2}
\mathcal{X} =\varPi _1\ast \left[ \begin{matrix}
	\mathcal{I} _{r_1}&		\mathcal{O}\\
	\mathcal{M}&		\mathcal{O}\\
\end{matrix} \right] \ast \varPi _{1}^{\top},\quad \mathcal{Y} =\varPi _{2}^{\top}\ast \left[ \begin{matrix}
	\mathcal{I} _{r_2}&		\mathcal{H}\\
	\mathcal{O}&		\mathcal{O}\\
\end{matrix} \right] \ast \varPi _2,
\end{align}
where $\mathcal{O}$ denotes a tensor whose entries are all zero, and henceforth, all occurrences of $\mathcal{O}$ in the following paper refer to this specific tensor for convenience. Then the following are equivalent. 
\begin{enumerate}[label=(\roman*),leftmargin= .5in]
\item $\mathcal{A}$ is co-$(r_1,r_2)$-separable.
\item $\mathcal{A}$ can be written as
\begin{align}\label{e3.3}
\mathcal{A} =\varPi _1\ast \left[ \begin{matrix}
	\mathcal{S}&		\mathcal{S} \ast \mathcal{H}\\
	\mathcal{M} \ast \mathcal{S}&		\mathcal{M} \ast \mathcal{S} \ast \mathcal{H}\\
\end{matrix} \right] \ast \varPi _2
\end{align}
\item $\mathcal{A}$ can be written as
\begin{align}\label{e3.4}
\mathcal{A} =\mathcal{X}\ast \mathcal{A} \ast \mathcal{Y}
\end{align}
\item $\mathcal{A}$ can be written as
\begin{align}\label{e3.5}
\mathcal{A} =\mathcal{X}\ast \mathcal{A},\quad \mathcal{A} = \mathcal{A}\ast \mathcal{Y}
\end{align}
\end{enumerate}
\end{proposition}
\begin{proof}
(i) $\Leftrightarrow $ (ii). Transforming \eqref{e3.3} into the Fourier domain like \eqref{e2.5}, we have the block diagonal form, and the diagonal block $\left(\widehat{\Pi }_1\right)_k=\left(\varPi _1\right)_{::1}$, $\left(\widehat{\Pi }_2\right)_k=\left(\varPi _2\right)_{::1}$, $k\in\{1:p\}$. We choose each $\left(\widehat{\Pi }_1\right)_k$ such that it moves $\left(\widehat{A}_k\right)_{\mathbf{I:}}$ to the first $r_1$ rows of $\widehat{A}_k$, and choose each $\left(\widehat{\Pi }_2\right)_k$ such that it moves $\left(\widehat{A}_k\right)_{\mathbf{:J}}$ to the first $r_2$ columns of $\widehat{A}_k$, which means $\left(\widehat{A}_k\right)_{\mathbf{I}\mathbf{J}}$ is the $r_1\times r_2$ block in top left of $\left(\widehat{\Pi}_1^\top\right)_k\widehat{A}_k\left(\widehat{\Pi}_2^\top\right)_k$. Restoring them into the tensor form, we have $\mathcal{A}_{\mathbf{I}\mathbf{J}:} = \left(\varPi_1^\top\ast \mathcal{A}\ast\varPi_2^\top\right)_{1:r_1,1:r_2,:}$. Therefore, let $\mathcal{A}_{\mathbf{IJ}:}=\mathcal{S}$, $\mathcal{P}_1=\varPi_1\ast\left[ \begin{matrix}
	\mathcal{I}_{r_1}\\
	\mathcal{M}\\
\end{matrix} \right]$, $\mathcal{P}_2 = \left[ \begin{matrix}
	\mathcal{I} _{r_2}&		\mathcal{H}\\
\end{matrix} \right] \ast \varPi _2$, then $\mathcal{A} = \mathcal{P}_1\ast\mathcal{A}_{\mathbf{IJ}:}\ast\mathcal{P}_2$, the equivalent characterization is proved.
\\
(ii) $\Leftrightarrow $ (iii). Let $\mathcal{A}'=\left[ \begin{matrix}
	\mathcal{S}&		\mathcal{S} \ast \mathcal{H}\\
	\mathcal{M} \ast \mathcal{S}&		\mathcal{M} \ast \mathcal{S} \ast \mathcal{H}\\
\end{matrix} \right]$, then $\mathcal{A}'=\varPi_1^\top\ast\mathcal{A}\ast\varPi_2^\top$. Note that
$\mathcal{A}' =\left[ \begin{matrix}
	\mathcal{I} _{r_1}&		\mathcal{O}\\
	\mathcal{M}&		\mathcal{O}\\
\end{matrix} \right] \ast \mathcal{A}'\ast\left[ \begin{matrix}
	\mathcal{I} _{r_2}&		\mathcal{H}\\
	\mathcal{O}&		\mathcal{O}\\
\end{matrix} \right],$ therefore
$$\mathcal{A} =\varPi_1\ast\left[ \begin{matrix}
	\mathcal{I} _{r_1}&		\mathcal{O}\\
	\mathcal{M}&		\mathcal{O}\\
\end{matrix} \right] \ast \varPi_1^\top\ast\mathcal{A}\ast\varPi_2^\top\ast\left[ \begin{matrix}
	\mathcal{I} _{r_2}&		\mathcal{H}\\
	\mathcal{O}&		\mathcal{O}\\
\end{matrix} \right]\ast\varPi_2=\mathcal{X}\ast\mathcal{A}\ast\mathcal{Y}.$$
\\
(ii) $\Rightarrow $ (iv). 
\begin{align*}
\mathcal{X} \ast \mathcal{A} &=\varPi _1\ast \left[ \begin{matrix}
	\mathcal{I} _{r_1}&		\mathcal{O}\\
	\mathcal{M}&		\mathcal{O}\\
\end{matrix} \right] \ast \varPi _{1}^{\top}\ast \varPi _1\ast \left[ \begin{matrix}
	\mathcal{S}&		\mathcal{S} \ast \mathcal{H}\\
	\mathcal{M} \ast \mathcal{S}&		\mathcal{M} \ast \mathcal{S} \ast \mathcal{H}\\
\end{matrix} \right] \ast \varPi _2
\\
&=\varPi _1\ast \left[ \begin{matrix}
	\mathcal{I} _{r_1}&		\mathcal{O}\\
	\mathcal{M}&		\mathcal{O}\\
\end{matrix} \right] \ast \left[ \begin{matrix}
	\mathcal{S}&		\mathcal{S} \ast \mathcal{H}\\
	\mathcal{M} \ast \mathcal{S}&		\mathcal{M} \ast \mathcal{S} \ast \mathcal{H}\\
\end{matrix} \right] \ast \varPi _2=\mathcal{A},
\end{align*}
and it is easy to prove $\mathcal{A} = \mathcal{A}\ast \mathcal{Y}$ in the same way.
\\
(iv) $\Rightarrow $ (iii). $\mathcal{A} =\mathcal{X}\ast \mathcal{A}$ and $\mathcal{A} = \mathcal{A}\ast \mathcal{Y}$ yield $\mathcal{A} =\mathcal{X}\ast \mathcal{A}=\mathcal{X}\ast \mathcal{A}\ast \mathcal{Y}$.
\end{proof}

\begin{remark}\label{t3.3}
Proposition \ref{t3.2} plays a crucial role in our analysis. Characterization (ii) reveals the structure of the coseparable tensor, aiding in the establishment of a connection between the co-($r_1,r_2$)-separable form and the t-CUR decomposition, as detailed in \Cref{t3.5}. Furthermore, characterizations (iii) and (iv) offer methods to identify the coseparable cores, outlined in \Cref{t3.6} and the subsequent discussions, thereby enriching our comprehension and practical implementation of these concepts.
\end{remark}

\begin{theorem}{\rm (Scaling)}
Tensor $\mathcal{A}\in\mathbb{R}_+^{m\times n\times p}$ is co-$\left(r_1, r_2\right)$-separable if and only if $\mathcal{A}^{\mathrm{scale}}=\mathcal{D}_r\ast\mathcal{A}\ast\mathcal{D}_c$ is co-$\left(r_1, r_2\right)$-separable for any $f$-diagonal tensors $\mathcal{D}_r\in\mathbb{R}_+^{m\times m\times p}$ and $\mathcal{D}_c\in\mathbb{R}_+^{n\times n\times p}$ whose diagonal entries of their first frontal slice are positive and the other frontal slices are 0.
\end{theorem}
\begin{proof}
Since $\mathcal{A}$ is co-($r_1, r_2$)-separable,   there exist index sets 
 $\mathbf{I}$, $\mathbf{J}$, along with  tensors $\mathcal{P}_1$ and $\mathcal{P}_2$, such that
$\mathcal{A} = \mathcal{P}_1\ast\mathcal{A}_{\mathbf{I}\mathbf{J}:}\ast\mathcal{P}_2$, where $|\mathbf{I}|=r_1$ and  $|\mathbf{J}|=r_2$. Consequently, we have 
\begin{align*}
\mathcal{A} ^{\mathrm{scale}}&=\mathcal{D} _r\ast \mathcal{A} \ast \mathcal{D} _c
\\
&=\mathcal{D} _r\ast \mathcal{P} _1\ast \mathcal{A} _{\mathbf{IJ}:}\ast \mathcal{P} _2\ast \mathcal{D} _c
\\
&=\left( \mathcal{D} _r\ast \mathcal{P} _1\ast \left( \mathcal{D} _{r}^{-1} \right) _{\mathbf{II}:} \right) \ast \left( \left( \mathcal{D} _r \right) _{\mathbf{II}:}\ast \mathcal{A} _{\mathbf{IJ}:}\ast \left( \mathcal{D} _c \right) _{\mathbf{JJ}:} \right) \ast \left( \left( \mathcal{D} _{c}^{-1} \right) _{\mathbf{JJ}:}\ast \mathcal{P} _2\ast \mathcal{D} _c \right) .
\end{align*}
Note that $\mathcal{D}_r$ and $\mathcal{D}_c$ are f-diagonal and $\left(\mathcal{P}_1\right)_{\mathbf{I}::}=\mathcal{I}_{r_1}$, $\left(\mathcal{P}_2\right)_{:\mathbf{J}:}=\mathcal{I}_{r_2}$,
$$\mathcal{A} _{\mathbf{IJ}:}^{\mathrm{scale}}=\left( \mathcal{D} _r \right) _{\mathbf{I}::}\ast \mathcal{A} \ast \left( \mathcal{D} _c \right) _{:\mathbf{J}:}=\left( \mathcal{D} _r \right) _{\mathbf{I}::}\ast \mathcal{P} _1\ast \mathcal{A} _{\mathbf{IJ}:}\ast \mathcal{P} _2\ast \left( \mathcal{D} _c \right) _{:\mathbf{J}:}=\left( \mathcal{D} _r \right) _{\mathbf{II}:}\ast \mathcal{A} _{\mathbf{IJ}:}\ast \left( \mathcal{D} _c \right) _{\mathbf{JJ}:}.$$
Denote $\mathcal{P} _{1}^{\mathrm{scale}}=\mathcal{D} _r\ast \mathcal{P} _1\ast \left( \mathcal{D} _{r}^{-1} \right) _{\mathbf{II}:}$ and $\mathcal{P} _{2}^{\mathrm{scale}}=\left( \mathcal{D} _{c}^{-1} \right) _{\mathbf{JJ}:}\ast \mathcal{P} _2\ast \mathcal{D} _c$, we have
$$\mathcal{A} ^{\mathrm{scale}}=\mathcal{P} _{1}^{\mathrm{scale}}\ast \mathcal{A} _{\mathbf{IJ}:}^{\mathrm{scale}}\ast \mathcal{P} _{2}^{\mathrm{scale}},$$
where
$$\left( \mathcal{P} _{1}^{\mathrm{scale}} \right) _{\mathbf{I}::}=\left( \mathcal{D} _r \right) _{\mathbf{I}::}\ast \mathcal{P} _1\ast \left( \mathcal{D} _{r}^{-1} \right) _{\mathbf{II}:}=\mathcal{I} _{r_1}
\quad
\left( \mathcal{P} _{2}^{\mathrm{scale}} \right) _{:\mathbf{J}:}=\left( \mathcal{D} _{c}^{-1} \right) _{\mathbf{JJ}:}\ast \mathcal{P} _2\ast \left( \mathcal{D} _c \right) _{:\mathbf{J}:}=\mathcal{I} _{r_2}.$$
Therefore $\mathcal{A} ^{\mathrm{scale}}$ is co-($r_1, r_2$)-separable.
\end{proof}

Note that a tensor can have different coseparabilities with different $\left(r_1,r_2\right)$, so it is important to find the minimal $\left(r_1,r_2\right)$ that provides the most compression of the tensor. Then we define the minimally co-$\left(r_1,r_2\right)$-separable tensor.

\begin{definition}\label{t3.4}
A tensor $\mathcal{A}\in\mathbb{R}^{m\times n\times p}_+$ is \rev{minimally} co-$(r_1,r_2)$-separable if $\mathcal{A}$ is co-$(r_1,r_2)$-separable and $\mathcal{A}$ is not co-$(r'_1,r'_2)$-separable for any $r'_1<r_1$ and $r'_2<r_2$.  
\end{definition}
\rev{\begin{remark} \label{newremark}
    For the case of minimally coseparable NMF, the solution \(A_{\mathbf{IJ}}\) is determined uniquely, modulo scaling and permutation.  Specifically, when the rank of matrix \(A\) satisfies \(\mathsf{rank}(A) = r_1 = r_2\), the solution tuple \((P_1, A_{\mathbf{IJ}}, P_2)\) is unique up to permutation and scaling, as stated in \cite[Theorem 5]{pan2021co}.
    For the minimally coseparable NTF, in scenarios where \(\left(\mathcal{P}_1\right)_{::,2:p}\) and \(\left(\mathcal{P}_2\right)_{::,2:p}\) are all zero, 
    \begin{align*}
        \mathsf{unfold}\left(\mathcal{A}\right)&=\mathsf{unfold}\left(\mathcal{P}_1\ast\mathcal{A}_{\mathbf{IJ:}}\ast\mathcal{P}_2\right)\\
        &=\mathsf{bcirc}\left(\mathcal{P}_1\right)\mathsf{bcirc}\left(\mathcal{A}_{\mathbf{IJ:}}\right)\mathsf{unfold}\left(\mathcal{P}_2\right)
        \\
        &=\left[ \begin{matrix}
	\left(\mathcal{P}_1\right) _{::1}& O&  \cdots& O\\
	O& \left(\mathcal{P}_1\right) _{::1}&  \cdots& O\\
	\vdots&    \vdots& \ddots& \vdots\\
	O& O&  \cdots& \left(\mathcal{P}_1\right) _{::1}\\
        \end{matrix} \right]\!\!
        \left[ \begin{matrix}
	\mathcal{A} _{\mathbf{IJ}1}&		\mathcal{A} _{\mathbf{IJ}p}&		\cdots&		    \mathcal{A} _{\mathbf{IJ}2}\\
	\mathcal{A} _{\mathbf{IJ}2}&		\mathcal{A} _{\mathbf{IJ}1}&		\cdots&		    \mathcal{A} _{\mathbf{IJ}3}\\
	\vdots&		\vdots&		\ddots&		\vdots\\
	\mathcal{A} _{\mathbf{IJ}p}&		\mathcal{A} _{\mathbf{IJ},p-1}&		\cdots&		\mathcal{A} _{\mathbf{IJ}1}\\
        \end{matrix} \right]\!\!
        \left[ \begin{array}{c}
	\left(\mathcal{P}_2\right) _{::1}\\
	O\\
	\vdots\\
	O\\
        \end{array} \right]
        \\
        &=\left[ \begin{array}{c}
	\left(\mathcal{P}_1\right) _{::1}\mathcal{A} _{\mathbf{IJ}1}\\
	\left(\mathcal{P}_1\right) _{::1}\mathcal{A} _{\mathbf{IJ}2}\\
	\vdots\\
	\left(\mathcal{P}_1\right) _{::1}\mathcal{A} _{\mathbf{IJ}p}\\
        \end{array} \right]\left(\mathcal{P}_2\right) _{::1},
    \end{align*}
    which means \(\mathsf{unfold}\left(\mathcal{A}\right)\) is a minimally \(r_2\) separable matrix \cite{separability} and \(\mathcal{A}_{:\mathbf{J}:}=\mathsf{fold}\left(\mathsf{unfold}\left(\mathcal{A}\right)_{:\mathbf{J}}\right)\) is unique up to scaling and permutation according to the proof of \cite[Theorem 5]{pan2021co}. Analogously, \(\mathsf{unfold\left(\mathcal{A}^\top\right)}\) is minimally \(r_1\) separable and \(\mathcal{A}_{\mathbf{I}::}\) is unique up to scaling and permutation. Then it is easy to derive the uniqueness of the core \(\mathcal{A}_{\mathbf{IJ}:}\), modulo scaling and permutation. However, for the general \(\mathcal{P}_1\) and \(\mathcal{P}_2\), the proof of the uniqueness of the core \(\mathcal{A}_{\mathbf{IJ}:}\) remains a task for future research. 
    We will now show that the solution tuple \((\mathcal{P}_1, \mathcal{A}_{\mathbf{IJ}:}, \mathcal{P}_2)\) of coseparable NTF is not  unique, even if the tensor tubalrank i.e., \(\mathsf{rank}_t(\mathcal{A})\) equals \(r_1 = r_2\). To illustrate this point, we present an example as follows.
\end{remark} }
 
\begin{example}
Let $\mathcal{A}\in\mathbb{R}_+^{3\times3\times2}$ be a minimum co-$(2,2)$-separable tensor,
\begin{align*}
    \mathcal{A} _{::1}=\left[ \begin{matrix}
	1&		2&		7\\
	3&		4&		11\\
	8&		10&		18\\
\end{matrix} \right] ,\quad \mathcal{A} _{::2}=\left[ \begin{matrix}
	5&		6&		7\\
	7&		8&		11\\
	8&		10&		18\\
\end{matrix} \right] ,
\end{align*}
and let $\mathcal{S}=\mathcal{A}_{1:2,1:2,:}$, i.e.,
\begin{align*}
\mathcal{S} _{::1}=\left[ \begin{matrix}
	1&		2\\
	3&		4\\
\end{matrix} \right] ,\quad \mathcal{S} _{::2}=\left[ \begin{matrix}
	5&		6\\
	7&		8\\
\end{matrix} \right] .
\end{align*}
For all $0\leqslant a,b\leqslant 1$, $\mathcal{P}_1$ and $\mathcal{P}_2$ can be
$$\left( \mathcal{P} _1 \right) _{::1}=\left[ \begin{matrix}
	1&		0\\
	0&		1\\
	1-a&		a\\
\end{matrix} \right] ,\quad \left( \mathcal{P} _1 \right) _{::2}=\left[ \begin{matrix}
	0&		0\\
	0&		0\\
	a&		1-a\\
\end{matrix} \right] ,\quad $$
$$\left( \mathcal{P} _2 \right) _{::1}=\left[ \begin{matrix}
	1&		0&		1-b\\
	0&		1&		b\\
\end{matrix} \right] ,\quad \left( \mathcal{P} _2 \right) _{::2}=\left[ \begin{matrix}
	0&		0&		b\\
	0&		0&		1-b\\
\end{matrix} \right]. $$
\end{example}

\begin{theorem}{\rm{(Relationship between Coseparable NTF and t-CUR decomposition)}}\label{t3.5}
	\begin{enumerate}[label=(\roman*),leftmargin=.5in]
    \item \label{thm:coNTF2TCUR}
 Any co-$(r_1,r_2)$-separable tensor $\mathcal{A}\in\mathbb{R}^{m\times n\times p}_+$ admits an exact t-CUR decomposition.
\item \label{thm:TCUR2coNTF}
Let tensor $\mathcal{A}\in\mathbb{R}^{m\times n\times p}_+$ admits a t-CUR decomposition  $\mathcal{A}=\mathcal{A}_{:\mathbf{J}:}\ast\mathcal{A}_{\mathbf{IJ}:}^\dag\ast\mathcal{A}_{\mathbf{I}::}$ with $|\mathbf{I}|=r_1$, $|\mathbf{J}|=r_2$. If tensor $\mathcal{A} _{\mathbf{I}^{\complement}\mathbf{J}:}\ast\mathcal{A}_{\mathbf{IJ}:}^\dag$ and $\mathcal{A}_{\mathbf{IJ}:}^\dag \ast \mathcal{A} _{\mathbf{I}\mathbf{J}^{\complement}:}$ are both nonnegative, then $\mathcal{A}$ is a co-$(r_1,r_2)$-separable tensor with core $\mathcal{A}_{\mathbf{I}\mathbf{J}:}$, where $\mathbf{I}^{\complement}$ and $\mathbf{J}^{\complement}$ are the complement sets of $\mathbf{I}$ and $\mathbf{J}$ respectively.
\end{enumerate}
\end{theorem}
\begin{proof}
\begin{enumerate}[label=(\roman*),leftmargin=.5in]
    \item  From Proposition \ref{t3.2} (ii), $\mathcal{A}$ can be written as
$$\mathcal{A} =\varPi _1\ast \left[ \begin{matrix}
	\mathcal{S}&		\mathcal{S} \ast \mathcal{H}\\
	\mathcal{M} \ast \mathcal{S}&		\mathcal{M} \ast \mathcal{S} \ast \mathcal{H}\\
\end{matrix} \right] \ast \varPi _2.$$
Notice that ${\varPi }_1$ moves horizontal slices with indices $\mathbf{I}$ to the first $r_1$ horizontal slices, and ${\varPi }_2$ moves lateral slices with indices $\mathbf{J}$ to the first $r_2$ lateral slices.  Using \Cref{t2.5}, we have
\rev{
\begin{align*}
\mathcal{A} =&\varPi _1\ast \left[ \begin{matrix}
	\mathcal{S}&		\mathcal{S} \ast \mathcal{H}\\
	\mathcal{M} \ast \mathcal{S}&		\mathcal{M} \ast \mathcal{S} \ast \mathcal{H}\\
\end{matrix} \right] \ast \varPi _2
\\
=&\varPi _1\ast \left[ \begin{array}{c}
	\mathcal{I}_{r_1}\\
	\mathcal{M}\\
\end{array} \right] \ast \mathcal{S}\ast  \left[ \begin{matrix}
	\mathcal{I}_{r_2}&		\mathcal{H}\\
\end{matrix} \right] \ast \varPi _2
\\
=&\varPi _1\ast \left[ \begin{array}{c}
	\mathcal{I}_{r_1}\\
	\mathcal{M}\\
\end{array} \right] \ast \mathcal{S}\ast \varPi _2\ast \varPi _2^\top\ast \mathcal{S} ^{\dagger}\ast \varPi _1^\top\ast \varPi _1\ast \mathcal{S}\ast \left[ \begin{matrix}
	\mathcal{I}_{r_2}&		\mathcal{H}\\
\end{matrix} \right] \ast \varPi _2
\\
=&\varPi _1\ast \left[ \begin{array}{c}
	\mathcal{S}\\
	\mathcal{M} \ast \mathcal{S}\\
\end{array} \right]\ast \varPi _2\ast \varPi _2^\top \ast \mathcal{S} ^{\dagger}\ast \varPi _1^\top\ast \varPi _1\ast \left[ \begin{matrix}
	\mathcal{S}&		\mathcal{S} \ast \mathcal{H}\\
\end{matrix} \right] \ast \varPi _2
\\
=&\mathcal{A} _{:\mathbf{J}:}\ast \varPi _2^\top \ast \mathcal{S} ^{\dagger}\ast \varPi _1^\top\ast \mathcal{A} _{\mathbf{I}::}=\mathcal{A} _{:\mathbf{J}:}\ast \mathcal{A} _{\mathbf{IJ}:}^{\dagger}\ast \mathcal{A} _{\mathbf{I}::},
\end{align*}}
which is a t-CUR decomposition of $\mathcal{A}$.
\item Since $\mathcal{A}$ admits $\mathcal{A}=\mathcal{A}_{:\mathbf{J}:}\ast\mathcal{A}_{\mathbf{IJ}:}^\dag\ast\mathcal{A}_{\mathbf{I}::}$, we can adjust the index sets $\mathbf{I}$ and $\mathbf{J}$ such that
$$\mathcal{A} _{\mathbf{IJ}:}=\mathcal{A} _{\mathbf{IJ}:}\ast \mathcal{A}_{\mathbf{IJ}:}^\dag \ast \mathcal{A} _{\mathbf{IJ}:},\quad \mathcal{A} _{\mathbf{I}\mathbf{J}^{\complement}:}=\mathcal{A} _{\mathbf{IJ}:}\ast \mathcal{A}_{\mathbf{IJ}:}^\dag \ast \mathcal{A} _{\mathbf{I}\mathbf{J}^{\complement}:},$$
$$\mathcal{A} _{\mathbf{I}^{\complement}\mathbf{J}:}=\mathcal{A} _{\mathbf{I}^{\complement}\mathbf{J}:}\ast \mathcal{A}_{\mathbf{IJ}:}^\dag \ast \mathcal{A} _{\mathbf{IJ}:},\quad \mathcal{A} _{\mathbf{I}^{\complement}\mathbf{J}^{\complement}:}=\mathcal{A} _{\mathbf{I}^{\complement}\mathbf{J}:}\ast \mathcal{A}_{\mathbf{IJ}:}^\dag \ast \mathcal{A} _{\mathbf{I}\mathbf{J}^{\complement}:}.$$
Let
$$\mathcal{S} =\mathcal{A} _{\mathbf{IJ}:},\quad \mathcal{M} =\mathcal{A} _{\mathbf{I}^{\complement}\mathbf{J}:}\ast \mathcal{A}_{\mathbf{IJ}:}^\dag ,\quad \mathcal{H} =\mathcal{A}_{\mathbf{IJ}:}^\dag \ast \mathcal{A} _{\mathbf{I}\mathbf{J}^{\complement}:},$$
then
$$\mathcal{A} =\varPi _1\ast \left[ \begin{matrix}
	\mathcal{A} _{\mathbf{IJ}:}&		\mathcal{A} _{\mathbf{I}\mathbf{J}^{\complement}:}\\
	\mathcal{A} _{\mathbf{I}^{\complement}\mathbf{J}:}&		\mathcal{A} _{\mathbf{I}^{\complement}\mathbf{J}^{\complement}:}\\
\end{matrix} \right] \ast \varPi _2=\varPi _1\ast \left[ \begin{matrix}
	\mathcal{S}&		\mathcal{S} \ast \mathcal{H}\\
	\mathcal{M} \ast \mathcal{S}&		\mathcal{M} \ast \mathcal{S} \ast \mathcal{H}\\
\end{matrix} \right] \ast \varPi _2.$$
Therefore, $\mathcal{A}$ is co-$(r_1,r_2)$-separable.
\end{enumerate}
\end{proof}

The statement \ref{thm:TCUR2coNTF} of Theorem \ref{t3.5} indicates that not every t-CUR decomposition yields a minimum coseparability, which means the reverse of statement \ref{thm:coNTF2TCUR} does not hold. To demonstrate this, we provide the following example.
\begin{example}
Let \(\mathcal{A}\in\mathbb{R}_+^{3\times3\times2}\) be a minimum co-\((2,2)\)-separable tensor, defined as:
\begin{align*}
    \mathcal{A}_{::1} = \begin{bmatrix}
	1 & 1 & \frac{3}{2} \\
	0 & 1 & \frac{3}{2} \\
	1 & 1 & \frac{3}{2} \\
\end{bmatrix}, \quad \mathcal{A}_{::2} = \begin{bmatrix}
	0 & 0 & \frac{1}{2} \\
	0 & 0 & \frac{1}{2} \\
	0 & 0 & \frac{1}{2} \\
\end{bmatrix}.
\end{align*}
Consider index sets \(\mathbf{I} = \{1, 2\}\), \(\mathbf{J} = \{1, 2\}\), and \(\tilde{\mathbf{I}} = \{1, 2\}\), \(\tilde{\mathbf{J}} = \{1, 3\}\). With these sets, \(\mathcal{A}\) admits two distinct t-CUR decompositions: \(\mathcal{A} = \mathcal{A}_{:\mathbf{J}:} \ast \mathcal{A}_{\mathbf{IJ}:}^\dag \ast \mathcal{A}_{\mathbf{I}::}\) and \(\mathcal{A} = \mathcal{A}_{:\tilde{\mathbf{J}}:} \ast \mathcal{A}_{\tilde{\mathbf{I}}\tilde{\mathbf{J}}:}^\dag \ast \mathcal{A}_{\tilde{\mathbf{I}}::}\).

For \(\mathcal{A}_{\mathbf{IJ}:}\), we have:
\begin{align*}
(\mathcal{P}_1)_{::1} = \begin{bmatrix}
	1 & 0 \\
	0 & 1 \\
	1 & 0 \\
\end{bmatrix}, \quad (\mathcal{P}_1)_{::2} = \begin{bmatrix}
	0 & 0 \\
	0 & 0 \\
	0 & 0 \\
\end{bmatrix}, \\
(\mathcal{P}_2)_{::1} = \begin{bmatrix}
	1 & 0 & 0 \\
	0 & 1 & \frac{3}{2} \\
\end{bmatrix}, \quad (\mathcal{P}_2)_{::2} = \begin{bmatrix}
	0 & 0 & 0 \\
	0 & 0 & \frac{1}{2} \\
\end{bmatrix}
\end{align*}
such that \(\mathcal{A} = \mathcal{P}_1 \ast \mathcal{A}_{\mathbf{I}\mathbf{J}:} \ast \mathcal{P}_2\). However, for \(\mathcal{A}_{\tilde{\mathbf{I}}\tilde{\mathbf{J}}:}\), it is not possible to identify nonnegative tensors  \(\mathcal{P}_1\) and \(\mathcal{P}_2\) that can ensure a coseparable NTF.
\end{example}

\section{Coseparable core selection algorithms}\label{sec:index-selection}
\subsection{Alternating CoS-NTF index selection}
  \Cref{t3.2} indicates that finding a minimal core $\mathcal{A}_{\mathbf{IJ}:}$ is equivalent to finding $\mathcal{X}\in \mathbb{R}^{m\times m\times p}_+$ and $\mathcal{Y} \in \mathbb{R}^{n\times n\times p}_+$ such that \eqref{e3.5} is satisfied and the number of nonzero lateral slice of $\mathcal{X}$ and nonzero horizontal slice of $\mathcal{Y}$ are minimized, which is stated as the following theorem.
\begin{theorem}\label{t3.6}
Let $\mathcal{A}$ be \rev{minimally} co-$(r_1, r_2)$-separable, and let $(\mathcal{X}^\mathrm{opt},\mathcal{Y}^\mathrm{opt} )$ be an optimal solution to the following optimization problem: 
\begin{equation}\label{e3.6}
\begin{aligned}
\min_{\mathcal{X} \in \mathbb{R} _{+}^{m\times m\times p}\atop
\mathcal{Y} \in \mathbb{R} _{+}^{n\times n\times p}}\ &\left\| \mathcal{X} \right\| _{\mathrm{lat},0}+\left\| \mathcal{Y} \right\| _{\mathrm{hor},0}
\\
\mathrm{s.t.}\ &\mathcal{A} =\mathcal{X} \ast \mathcal{A} ,\ \mathcal{A} =\mathcal{A} \ast \mathcal{Y} ,
\end{aligned}
\end{equation}
where $\left\| \cdot\right\|_{\mathrm{lat},0}$ and $\left\| \cdot\right\|_{\mathrm{hor},0}$ denote the number of nonzero lateral and horizontal slices, respectively. Let $\mathbf{I}$, $\mathbf{J}$ be the index sets of the nonzero lateral slices of $\mathcal{X}^\mathrm{opt}$ and the nonzero horizontal slices of $\mathcal{Y}^\mathrm{opt}$. Then $\left|\mathbf{I}\right|=r_1$, $\left|\mathbf{J}\right|=r_2$ and $\mathcal{A}_{\mathbf{IJ}:}$ is the minimal core of $\mathcal{A}$.
\end{theorem}
\begin{proof}
Since $\mathcal{A}$ is \rev{minimally} co-$(r_1, r_2)$-separable, \Cref{t3.5} implies the existence of nonnegative tensors $\mathcal{X}$ and $\mathcal{Y}$ such that $\mathcal{A} = \mathcal{X}\ast\mathcal{A}$ and $\mathcal{A} = \mathcal{A}\ast\mathcal{Y}$, where $\mathcal{X}$ and $\mathcal{Y}$ satisfy \eqref{e3.2}. By the optimality of $(\mathcal{X}^\mathrm{opt},\mathcal{Y}^\mathrm{opt} )$, we have $\left|\mathbf{I}\right|=\left\| \mathcal{X}^\mathrm{opt}\right\|_{\mathrm{lat},0}\leqslant \left\| \mathcal{X}\right\|_{\mathrm{lat},0}=r_1$ and $\left|\mathbf{J}\right|=\left\| \mathcal{Y}^\mathrm{opt}\right\|_{\mathrm{hor},0}\leqslant \left\| \mathcal{Y}\right\|_{\mathrm{hor},0}=r_2$.

On the other hand, we assume that $\mathcal{X}^\mathrm{opt}$ and $\mathcal{Y}^\mathrm{opt}$ are of the forms 
$$\mathcal{X}^\mathrm{opt}=\left[ \begin{matrix}
	\mathcal{X} _{11}^\mathrm{opt}&		\mathcal{O}\\
	\mathcal{X} _{21}^\mathrm{opt}&		\mathcal{O}\\
\end{matrix} \right]\quad \mathcal{Y}^\mathrm{opt}=\left[ \begin{matrix}
	\mathcal{Y} _{11}^\mathrm{opt}&		\mathcal{Y} _{12}^\mathrm{opt}\\
	\mathcal{O}&		\mathcal{O}\\
\end{matrix} \right], $$
where $\mathcal{X} _{11}^\mathrm{opt}\in \mathbb{R}^{\left|\mathbf{I}\right| \times \left|\mathbf{I}\right|\times p}_+ $, $\mathcal{Y} _{11}^\mathrm{opt}\in \mathbb{R}^{\left|\mathbf{J}\right| \times \left|\mathbf{J}\right|\times p}_+ $. According to the constraint,
$$\mathcal{A} =\left[ \begin{matrix}
	\mathcal{A} _{11}&		\mathcal{A} _{12}\\
	\mathcal{A} _{21}&		\mathcal{A} _{22}\\
\end{matrix} \right] =\mathcal{X} ^\mathrm{opt}\ast \mathcal{A} =\left[ \begin{matrix}
	\mathcal{X} _{11}^\mathrm{opt}\ast \mathcal{A} _{11}&		\mathcal{X} _{11}^\mathrm{opt}\ast \mathcal{A} _{12}\\
	\mathcal{X} _{21}^\mathrm{opt}\ast \mathcal{A} _{11}&		\mathcal{X} _{21}^\mathrm{opt}\ast \mathcal{A} _{12}\\
\end{matrix} \right] =\mathcal{A} \ast \mathcal{Y} ^\mathrm{opt}=\left[ \begin{matrix}
	\mathcal{A} _{11}\ast \mathcal{Y} _{11}^\mathrm{opt}&		\mathcal{A} _{11}\ast \mathcal{Y} _{12}^\mathrm{opt}\\
	\mathcal{A} _{21}\ast \mathcal{Y} _{11}^\mathrm{opt}&		\mathcal{A} _{21}\ast \mathcal{Y} _{12}^\mathrm{opt}\\
\end{matrix} \right] ,$$
then
$$\mathcal{A} =\left[ \begin{matrix}
	\mathcal{A} _{11}&		\mathcal{X} _{11}^\mathrm{opt}\ast \mathcal{A} _{12}\\
	\mathcal{X} _{21}^\mathrm{opt}\ast \mathcal{A} _{11}&		\mathcal{X} _{21}^\mathrm{opt}\ast \mathcal{A} _{11}\ast \mathcal{Y} _{12}^\mathrm{opt}\\
\end{matrix} \right] =\left[ \begin{array}{c}
	\mathcal{I} _{\left| \mathbf{I} \right|}\\
	\mathcal{X} _{21}^\mathrm{opt}\\
\end{array} \right] \ast \mathcal{A} _{11}\ast \left[ \begin{matrix}
	\mathcal{I} _{\left| \mathbf{J} \right|}&		\mathcal{Y} _{12}^\mathrm{opt}\\
\end{matrix} \right] ,$$
which means $\mathcal{A}$ is co-($\left| \mathbf{I} \right|,\left| \mathbf{J} \right|$) separable and $\left|\mathbf{I}\right|\geqslant r_1$, $\left|\mathbf{J}\right|\geqslant r_2$.

Therefore, $\left|\mathbf{I}\right|=r_1$, $\left|\mathbf{J}\right|=r_2$ and $\mathcal{A}_{\mathbf{IJ}:}$ is the minimal core.
\end{proof}

The model \eqref{e3.6} can be separated into two independent problems as follows.
\begin{equation}\label{e3.7}
\begin{aligned}
\min_{\mathcal{X} \in \mathbb{R} _{+}^{m\times m\times p}}\ &\left\| \mathcal{X} \right\| _{\mathrm{lat},0}\quad \mathrm{s.t.}\ \mathcal{A}=\mathcal{X}\ast \mathcal{A} ,
\\
\min_{\mathcal{Y} \in \mathbb{R} _{+}^{n\times n\times p}}\ &\left\| \mathcal{Y} \right\| _{\mathrm{hor},0}\quad \mathrm{s.t.}\ \mathcal{A}=\mathcal{A}\ast \mathcal{Y} .
\end{aligned}
\end{equation}
 \Cref{t3.6} confirms that the optimal solution \((\mathcal{X}^\mathrm{opt}, \mathcal{Y}^\mathrm{opt})\) adheres to the form specified in \eqref{e3.2}. Specifically, \eqref{e3.2} suggests that the nonzero slices of \(\mathcal{X}^\mathrm{opt}\) and \(\mathcal{Y}^\mathrm{opt}\) are identified by the identity tensor blocks \(\mathcal{I}_{r_1}\) and \(\mathcal{I}_{r_2}\). These blocks are characterized by having all entries as zero, except in the first frontal slice. This results in \(\left\| \mathcal{X}^\mathrm{opt} \right\|_{\mathrm{lat},0} = \left\| \mathcal{X}^\mathrm{opt}_{::1} \right\|_{\mathrm{col},0} = r_1\) and \(\left\| \mathcal{Y}^\mathrm{opt} \right\|_{\mathrm{hor},0} = \left\| \mathcal{Y}^\mathrm{opt}_{::1} \right\|_{\mathrm{row},0} = r_2\). Note that our goal is finding which slices of $\mathcal{A}$ are more crucial and the model \eqref{e3.7} aims to determine index sets $\mathbf{I,J}$ to form the coseparable core $\mathcal{A}_{\mathbf{IJ}:}$. This understanding prompts us to focus primarily on the first frontal slice and devise a more straightforward solution. 
\begin{align*}
\min_{\mathcal{X} \in \mathbb{R} _{+}^{m\times m\times p}}\left\| \mathcal{X}_{::1} \right\| _{\mathrm{col},0}\quad &\mathrm{s.t.}\ \mathcal{A}=\mathcal{X}\ast \mathcal{A} ,
\\
\min_{\mathcal{Y}\in\mathbb{R}_{+}^{n\times n\times p}}\left\|\mathcal{Y}_{::1}\right\|_{\mathrm{row},0}\quad &\mathrm{s.t.}\ \mathcal{A}=\mathcal{A}\ast \mathcal{Y}.
\end{align*}

Subsequently, we will address the second challenge, with a similar approach applicable to the first. Consider a tensor $\mathcal{Y}'$ such that $\mathcal{Y}'_{::1} =\Pi _{2}^{\top}\left[ \begin{matrix}
	I_{r_2}&		H'\\
	O&		O\\
\end{matrix} \right] \Pi _2$ and $\mathcal{Y}'_{::,2:p}=O$, then $\mathcal{Y}'$ has the form of $\varPi _{2}^{\top}\ast \left[ \begin{matrix}
	\mathcal{I} _{r_2}&		\mathcal{H}\\
	\mathcal{O}& \mathcal{O}\\
\end{matrix} \right] \ast \varPi _2$. If $\mathcal{Y}'$ satisfies $\mathcal{A}=\mathcal{A}\ast\mathcal{Y}'$, we have 
\begin{align*}
    \mathsf{unfold}\left( \mathcal{A} \right) &=\mathsf{bcirc}\left( \mathcal{A} \right) \mathsf{unfold}\left( \mathcal{Y} ' \right) 
\\
&=\left[ \begin{matrix}
	\mathcal{A} _{::1}&		\mathcal{A} _{::p}&		\cdots&		\mathcal{A} _{::2}\\
	\mathcal{A} _{::2}&		\mathcal{A} _{::1}&		\cdots&		\mathcal{A} _{::3}\\
	\vdots&		\vdots&		\ddots&		\vdots\\
	\mathcal{A} _{::p}&		\mathcal{A} _{::,p-1}&		\cdots&		\mathcal{A} _{::1}\\
\end{matrix} \right] \left[ \begin{array}{c}
	\mathcal{Y} '_{::1}\\
	O\\
	\vdots\\
	O\\
\end{array} \right] =\left[ \begin{array}{c}
	\mathcal{A} _{::1}\mathcal{Y} '_{::1}\\
	\mathcal{A} _{::2}\mathcal{Y} '_{::1}\\
	\vdots\\
	\mathcal{A} _{::p}\mathcal{Y} '_{::1}\\
\end{array} \right] =\mathsf{unfold}\left( \mathcal{A} \right) \mathcal{Y} '_{::1}.
\end{align*}
Moreover, $\left\|\mathcal{Y}' \right\|_{\mathrm{hor},0}=\left\|\mathcal{Y}'_{::1}\right\|_{\mathrm{row},0}=r_2$, which means $\mathcal{Y}'$ is an optimal solution of \eqref{e3.7}. Therefore, we can find the index set $\mathbf{J}$ by solving the following model
\begin{equation}\label{e3.8}
\begin{aligned}
\min_{Y \in \mathbb{R} _{+}^{n\times n}} \ \left\| Y \right\| _{\mathrm{row},0}\quad \mathrm{s.t.}\ \mathsf{unfold}\left(\mathcal{A}\right)=\mathsf{unfold}\left(\mathcal{A}\right)Y.
\end{aligned}
\end{equation}

However, note that
$$\mathsf{rank}\left( \mathsf{unfold}\left( \mathcal{A} \right) \right) =\mathsf{rank}\left( \mathsf{bcirc}\left( \mathcal{A} \right) \right) =\mathsf{rank}\left( \widehat{A} \right) =\sum_{k=1}^p{\mathsf{rank}\left( \widehat{A}_k \right)}=\sum_{k=1}^p{r_k},$$
the sum of each entry of $\mathsf{rank}_m\left( \mathcal{A} \right)$, which means $\mathsf{unfold}\left(\mathcal{A}\right)$ may have full rank even if $\mathcal{A}$ doesn't have full rank. Besides, the presence of noise in $\mathcal{A}$ can lead to a full rank of $\mathsf{unfold}\left(\mathcal{A}\right)$, thereby resulting in a solution $Y = I_n$. To avoid this situation, we will consider the following problem
\begin{equation}\label{e3.9}
\begin{aligned}
\min_{Y \in \mathbb{R} _{+}^{n\times n}} \ \left\| Y \right\| _{\mathrm{row},0}\quad \mathrm{s.t.}\ \left\|\mathsf{unfold}\left(\mathcal{A}\right)-\mathsf{unfold}\left(\mathcal{A}\right)Y\right\|_{\fro}\leqslant \epsilon.
\end{aligned}
\end{equation} 

The model \eqref{e3.9} can be relaxed to the following model as demonstrated in \cite{pan2021co}.
\begin{equation}\label{e3.10}
\begin{aligned}
\min_{Y \in \Omega_Y} \ \frac{1}{2}\left\|\mathsf{unfold}\left(\mathcal{A}\right)-\mathsf{unfold}\left(\mathcal{A}\right)Y\right\|^2_{\fro}+\lambda_Y\mathsf{tr}\left( Y \right),
\end{aligned}
\end{equation} 
where $\Omega_Y=\left\{Y\in\mathbb{R}_+^{n\times n}:0\leqslant Y_{ij}\leqslant1,\ \omega_iY_{ij}\leqslant\omega_jY_{ii},\ 1\leqslant i,j\leqslant n,\ \omega_j=\left\|\mathsf{unfold}\left(\mathcal{A}\right)_{:j}\right\|_1\right\}$, $\lambda_Y$ is the regularization parameter to balance two terms. We can solve this problem and select indices by the fast
gradient method for separable-NMF (SNMF-FGM) algorithm \cite{SNMF-FGM}.

For the first problem of \eqref{e3.7}, we can analogously solve the following problem to select indices
\begin{align}\label{e3.11}
    \min_{X \in \Omega_X} \ \frac{1}{2}\left\|\mathsf{unfold}\left(\mathcal{A}^\top\right)-\mathsf{unfold}\left(\mathcal{A}^\top\right)X^\top\right\|^2_{\fro}+\lambda_X\mathsf{tr}\left( X \right),
\end{align}
where $\Omega_X=\left\{X\in\mathbb{R}_+^{m\times m}:0\leqslant X_{ij}\leqslant1,\ \omega_iX_{ji}\leqslant\omega_jX_{ii},\ 1\leqslant i,j\leqslant n,\ \omega_j=\left\|\mathsf{unfold}\left(\mathcal{A}^\top\right)_{:j}\right\|_1\right\}$.

So we use the SNMF-FGM algorithm to alternately solve two optimization problems of $X$ and $Y$ and find $\mathbf{I}$ and $\mathbf{J}$ iteratively as outlined in \Cref{a2}. Besides, notice that we derive \cref{e3.8} by letting the tensor \(\mathcal{Y}\) in \cref{e3.7} with only the first frontal slice nonzero, and we have already shown in \cref{newremark} that the solution is unique in this case, which guarantees a unique solution of \cref{a2}. 
\begin{algorithm}
\caption{Alternating CoS-NTF index selection}
\begin{algorithmic}\label{a2}
\STATE $\mathbf{I}=\mathsf{SNMF}$-$\mathsf{FGM}\left(\mathsf{unfold}\left(\mathcal{A}^\top\right), r_1\right)$
\STATE $\mathbf{J}=\mathsf{SNMF}$-$\mathsf{FGM}\left(\mathsf{unfold}\left(\mathcal{A}\left(\mathbf{I},:,:\right)\right), r_2\right)$
\FOR {$i=1:maxiter$}
\STATE $\mathcal{A}^{old}_\mathbf{I}=\mathcal{A}\left(\mathbf{I},:,:\right)$,\quad $\mathcal{A}^{old}_\mathbf{J}=\mathcal{A}\left(:,\mathbf{J},:\right)$
\STATE $\mathbf{I}=\mathsf{SNMF}$-$\mathsf{FGM}\left(\mathsf{unfold}\left(\mathcal{A}\left(:,\mathbf{J},:\right)^\top\right), r_1\right)$
\STATE $\mathbf{J}=\mathsf{SNMF}$-$\mathsf{FGM}\left(\mathsf{unfold}\left(\mathcal{A}\left(\mathbf{I},:,:\right)\right), r_2\right)$
\IF {$\left\|\mathcal{A}^{\mathrm{old}}_\mathbf{I}-\mathcal{A}\left(\mathbf{I},:,:\right)\right\|_{\fro}+\left\|\mathcal{A}^{\mathrm{old}}_\mathbf{J}-\mathcal{A}\left(:,\mathbf{J},:\right)\right\|_{\fro}\leqslant\delta$}
\STATE break
\ENDIF
\ENDFOR
\RETURN $\mathbf{I},\mathbf{J}$, the core $\mathcal{A}\left(\mathbf{I},\mathbf{J},:\right)$
\end{algorithmic}
\end{algorithm}

\subsection{t-CUR-DEIM index selection}
The initial model \eqref{e3.6} is to find the nonzero slices to form the core tensor $\mathcal{A}_{\mathbf{I}\mathbf{J}:}$. And the SNMF-FGM algorithm can decide which slices are the important ones. Since the relationship between the t-CUR decomposition and the coseparable NTF has been proved in Theorem \ref{t3.5}, we want to select the important slices with t-CUR sampling. First, we will introduce the tensor stable rank and the t-CUR sampling theorem.
\begin{definition}{\rm (Tensor stable rank)}\label{t4.1}
The stable rank of $\mathcal{A}\in \mathbb{R}^{m\times n\times p}$ is denoted as $\mathsf{st.rank}\left( \mathcal{A} \right)$ and
$$\mathsf{st.rank}\left( \mathcal{A} \right):=\frac{\left\|\mathcal{A}\right\|^2_{\fro}}{\left\|\mathcal{A}\right\|^2_2}. $$
\end{definition}

Let the tensor $\mathcal{A}\in \mathbb{R}^{m\times n\times p}$ have $\mathsf{rank}_t\left( \mathcal{A} \right)=r$. Using \eqref{e7}, we have the following result.
\begin{align}\label{e4.1}
\mathsf{st}.\mathsf{rank}\left( \mathcal{A} \right) :=\frac{\left\| \mathcal{A} \right\| _{\fro}^{2}}{\left\| \mathcal{A} \right\| _{2}^{2}}=\frac{\sum\limits_{k=1}^p{\left\| \widehat{A}_k \right\| _{\fro}^{2}}}{p\left\| \widehat{A} \right\| _{2}^{2}}=\frac{\sum\limits_{k=1}^p{\sum\limits_{i=1}^{r_k}{\sigma _{i}^{2}\left( \widehat{A}_k \right)}}}{p\sigma _{1}^{2}\left( \widehat{A} \right)}\leqslant \frac{\sum\limits_{k=1}^p{\sum\limits_{i=1}^r{\sigma _{1}^{2}\left( \widehat{A} \right)}}}{p\sigma _{1}^{2}\left( \widehat{A} \right)}=r=\mathsf{rank}_t\left( \mathcal{A} \right),
\end{align}
where $\widehat{A}$ satisfies \eqref{e2.4} and $r_k$ is the rank of diagonal block $\widehat{A}_k$.

\begin{theorem}{(t-CUR sampling)}\label{t4.2}
Let tensor $\mathcal{A}\in \mathbb{R}^{m\times n\times p}$ satisfy $\mathsf{st.rank}\left( \mathcal{A} \right)=s$, $\mathsf{rank}_t\left( \mathcal{A} \right)=r$. The index sets $\mathbf{I}$, $\mathbf{J}$ are sampled independently with replacement from $\left\{1:m\right\}$, $\left\{1:n\right\}$ according to two probability distributions $\mathbf{p}_i\geqslant\alpha^2_i\frac{\left\|\mathcal{A}_{i::}\right\|^2_{\fro}}{\left\|\mathcal{A}\right\|^2_{\fro}}$, $\mathbf{q}_j\geqslant\beta^2_j\frac{\left\|\mathcal{A}_{:j:}\right\|^2_{\fro}}{\left\|\mathcal{A}\right\|^2_{\fro}}$ respectively for some constants $\alpha_i,\beta_j>0$, where $\alpha_i=1$ if $\mathcal{A}_{i::}=0$ and $\beta_j=1$ if $\mathcal{A}_{:j:}=0$. Let $\delta \in \left( 0,1 \right) $, $\alpha=\min\limits_{i}\{\alpha_i\}$, $\beta=\min\limits_{j}\{\beta_j\}$, $\gamma=\min\{\alpha, \beta\}$, $\varepsilon \in\left(0,\min\left\{\sqrt{2\gamma}/\sqrt[4]{\delta},{\sigma _{\min}\left( \hat{A} \right)}/{\sigma _1\left( \hat{A} \right)}\right\}\right)$, tensor $\mathcal{R}=\mathcal{A}_{\mathbf{I}::}$, $\mathcal{C}=\mathcal{A}_{:\mathbf{J}:}$, $\mathcal{U}=\mathcal{A}_{\mathbf{I}\mathbf{J}:}$, and
$$\left| \mathbf{I} \right|,\left| \mathbf{J} \right|\geqslant C\frac{s}{\varepsilon ^4\delta}\log \left( \frac{s}{\varepsilon ^4\delta} \right) ,$$
then with probability at least $\left( 1-2\exp\left(-\frac{C_1 \alpha^2}{\delta}\right) \right)\left( 1-2\exp\left(-\frac{C_2\beta^2}{\delta}\right) \right) $, we have
$$\mathsf{rank}_m\left( \mathcal{A} \right)=\mathsf{rank}_m\left( \mathcal{U} \right),\quad\mathcal{A}=\mathcal{C}\ast \mathcal{U}^{\dag}\ast \mathcal{R}.$$
\end{theorem}
\begin{proof}
Under the assumptions above, the following estimation holds with probability at least $1-2\exp\left(-\frac{C_1\alpha^2}{\delta}\right)$, which has been proved in \cite[Lemma 3.2]{TCURproof}:
\begin{align}\label{e4.2}
\left\| \mathcal{A} ^{\top} \ast \mathcal{A}-\mathcal{R} ^{\top} \ast \mathcal{R} \right\| _2\leqslant \frac{\varepsilon ^2}{2}\left\| \mathcal{A} \right\| _{2}^{2}.
\end{align}
Using \eqref{e2.5} and \eqref{e4.2}, and noticing that the matrix spectral norm is unitarily invariant, we have
$$\left\| \hat{A}^{\top}\hat{A}-\hat{R}^{\top}\hat{R} \right\| _2=\left\| \mathcal{A} ^{\top}\ast \mathcal{A} -\mathcal{R} ^{\top}\ast \mathcal{R} \right\| _2\leqslant \frac{\varepsilon ^2}{2}\left\| \mathcal{A} \right\| _{2}^{2}=\frac{\varepsilon ^2}{2}\left\| \widehat{A} \right\| _{2}^{2}<\frac{1}{2}\sigma _{\min}^{2}\left( \widehat{A} \right) <\sigma _{\min}^{2}\left( \widehat{A} \right) .$$
Therefore, $\mathsf{rank}\left( \widehat{R} \right) \geqslant \mathsf{rank}\left( \widehat{A} \right) $. Since the horizontal slices of $\mathcal{R}$ are picked from $\mathcal{A}$, the corresponding rows of $\widehat{R}$ are picked from $\widehat{A}$ as well, then $\mathsf{rank}\left( \widehat{R} \right) \leqslant \mathsf{rank}\left( \widehat{A} \right) $, which yields $\mathsf{rank}\left( \widehat{R} \right) =\mathsf{rank}\left( \widehat{A} \right)$.

Besides, each diagonal block satisfies $\mathsf{rank}\left( \widehat{R}_k \right) \leqslant\mathsf{rank}\left( \widehat{A}_k \right)$ for  $k\in\{1:p\}$. Notice that 
$$\mathsf{rank}\left( \widehat{R} \right) =\sum_{k=1}^p{\mathsf{rank}\left( \widehat{R}_k \right)}\leqslant \sum_{k=1}^p{\mathsf{rank}\left( \widehat{A}_k \right)}=\mathsf{rank}\left( \widehat{A} \right) ,$$
hence each $\mathsf{rank}\left( \widehat{R}_k \right) = \mathsf{rank}\left( \widehat{A}_k \right)$, i.e., $\mathsf{rank}_m\left( \mathcal{R} \right) = \mathsf{rank}_m\left( \mathcal{A} \right)$.

Similarly, we can prove that $\mathsf{rank}_m\left( \mathcal{C} \right) = \mathsf{rank}_m\left( \mathcal{A} \right)$ with probability at least $1-2\exp\left(-\frac{C_2\beta^2}{\delta}\right)$. So $\mathsf{rank}_m\left( \mathcal{C} \right) = \mathsf{rank}_m\left( \mathcal{R} \right)=\mathsf{rank}_m\left( \mathcal{A} \right)$ with probability at least 
$$\left( 1-2\exp\left(-\frac{C_1 \alpha^2}{\delta}\right) \right)\left( 1-2\exp\left(-\frac{C_2\beta^2}{\delta}\right) \right) .$$ According to \Cref{t2.9}, it indicates that $\mathsf{rank}\left( \widehat{C}_k \right)=\mathsf{rank}\left( \widehat{R}_k \right) =\mathsf{rank}\left( \widehat{A}_k \right)$, which is equivalent to $\mathsf{rank}\left( \widehat{U}_k \right)=\mathsf{rank}\left( \widehat{A}_k \right)$ \cite[Theorem 5.5]{CUR}, $k\in\left\{1:p\right\}$, i.e., 
$$\mathsf{rank}_m\left( \mathcal{A} \right)=\mathsf{rank}_m\left( \mathcal{U} \right),\quad\mathcal{A}=\mathcal{C}\ast \mathcal{U}^{\dag}\ast \mathcal{R}.$$
\end{proof}

We can adjust each $\alpha_i$ and $\beta_j$ to reach different sampling distributions. If we let 
$\alpha _i=\frac{1}{\sqrt{m}}\mathop {\min} \limits_i\left\{ \frac{\lVert \mathcal{A} \rVert _{\fro}}{\lVert \mathcal{A}_{i::} \rVert _{\fro}} \right\}$ and $\beta _j=\frac{1}{\sqrt{n}}\mathop {\min} \limits_j\left\{ \frac{\lVert \mathcal{A} \rVert _{\fro}}{\lVert \mathcal{A}_{:j:} \rVert _{\fro}} \right\}$, we can prove that the uniform distributions
\begin{align}\label{e4.3}
\mathbf{p}_{i}^{\mathrm{unif}}=\frac{1}{m},\quad\mathbf{q}_{j}^{\mathrm{unif}}=\frac{1}{n},
\end{align}
satisfy Theorem \ref{t4.2}. When every $\alpha_i=\beta_j=1$, we obtain the slice-size sampling
\begin{align}\label{e4.4}
\mathbf{p}_{i}^{\mathrm{hor}}=\frac{\left\| \mathcal{A} _{i::} \right\| _{\fro}^{2}}{\left\| \mathcal{A} \right\| _{\fro}^{2}},\quad\mathbf{q}_{j}^{\mathrm{lat}}=\frac{\left\| \mathcal{A} _{:j:} \right\| _{\fro}^{2}}{\left\| \mathcal{A} \right\| _{\fro}^{2}}.
\end{align}
Let tensor $\mathcal{A}\in\mathbb{R}^{m\times n\times p}$ have $\mathsf{rank}_t\left(\mathcal{A}\right)=r$ and $\mathsf{st.rank}\left(\mathcal{A}\right)=s$. Suppose that $\mathcal{A}$ satisfies the truncated t-SVD
$$\mathcal{A}=\mathcal{W} _{:,1:r,:}\ast \varSigma _{1:r,1:r,:}\ast \mathcal{V} _{:,1:r,:}^{\top}.$$
Then 
\begin{align*}
\left\| \mathcal{A} _{i::} \right\| _{\fro}^{2}&=\left\| \left( \mathcal{W} _{:,1:r,:}\ast \varSigma _{1:r,1:r,:}\ast \mathcal{V} _{:,1:r,:}^{\top} \right) _{i::} \right\| _{\fro}^{2}
\\
&=\left\| \mathcal{W} _{i,1:r,:}\ast \varSigma _{1:r,1:r,:}\ast \mathcal{V} _{:,1:r,:}^{\top} \right\| _{\fro}^{2}=\left\| \mathcal{W} _{i,1:r,:}\ast \varSigma _{1:r,1:r,:} \right\| _{\fro}^{2},
\end{align*}
where the last equation is the orthogonal invariant property of the Frobenius norm. By transforming the equation into block diagonal form in the Fourier domain and using Remark \ref{t2.7}, we then have
\begin{align*}
\left\| \mathcal{A} _{i::} \right\| _{\fro}^{2}&=\frac{1}{p}\sum_{k=1}^p{\left\| \left( \widehat{W}_k\widehat{\Sigma }_k \right) _{i,1:r} \right\| _{\fro}^{2}}=\frac{1}{p}\sum_{k=1}^p{\sum_{j=1}^r{\left| \widehat{W}_k \right|_{ij}^{2}\left| \widehat{\Sigma }_k \right|_{jj}^{2}}}
\\
&\leqslant \frac{\sigma _{1}^{2}\left(\widehat{A}\right)}{p}\sum_{k=1}^p{\sum_{j=1}^r{\left| \widehat{W}_k \right|_{ij}^{2}}}=\frac{\sigma _{1}^{2}\left(\widehat{A}\right)}{p}\sum_{k=1}^p{\left\| \left( \widehat{W}_k \right) _{i,1:r} \right\| _{\fro}^{2}}=\left\|\mathcal{A}\right\|^2_2\left\| \mathcal{W} _{i,1:r,:} \right\| _{\fro}^{2}.
\end{align*}
Therefore, $\frac{1}{r}\left\| \mathcal{W} _{i,1:r,:} \right\| _{\fro}^{2}\geqslant \frac{1}{r}\frac{\left\| \mathcal{A} \right\| _{\fro}^{2}\left\| \mathcal{A} _{i::} \right\| _{\fro}^{2}}{\left\| \mathcal{A} \right\| _{\fro}^{2}\left\| \mathcal{A} \right\| _{2}^{2}}=\frac{s}{r}\frac{\left\| \mathcal{A} _{i::} \right\| _{\fro}^{2}}{\left\| \mathcal{A} \right\| _{\fro}^{2}}$ and $\frac{1}{r}\left\| \mathcal{V} _{j,1:r,:} \right\| _{\fro}^{2}\geqslant \frac{s}{r}\frac{\left\| \mathcal{A} _{:j:} \right\| _{\fro}^{2}}{\left\| \mathcal{A} \right\| _{\fro}^{2}}$.
Here we have leverage score sampling
\begin{align}\label{e4.5}
\mathbf{p}_{i}^{\mathrm{lvg},r}=\frac{1}{r}\left\| \mathcal{W} _{i,1:r,:} \right\| _{\fro}^{2},\quad\mathbf{q}_{j}^{\mathrm{lvg},r}=\frac{1}{r}\left\| \mathcal{V} _{j,1:r,:} \right\| _{\fro}^{2}.
\end{align}

With the above results, we have the t-CUR decomposition algorithms \Cref{a3}.
\begin{algorithm}
\caption{t-CUR decomposition}
\begin{algorithmic}\label{a3}
\STATE $\mathbf{Input}$: \ $\mathcal{A}\in \mathbb{R}^{m\times n\times p}$, sampling sizes $d_1,d_2$, sampling distributions $\mathbf{p}, \mathbf{q}$
\STATE Sample $d_1$ horizontal indices $\mathbf{I}(i)$ according to $\mathbf{p}$
\STATE Sample $d_2$ lateral indices $\mathbf{J}(j)$ according to $\mathbf{q}$
\STATE $\mathcal{C}=\mathcal{A}\left(:,\mathbf{J},:\right)$ 
\STATE $\mathcal{U}=\mathcal{A}\left(\mathbf{I},\mathbf{J},:\right)$ 
\STATE $\mathcal{R}=\mathcal{A}\left(\mathbf{I},:,:\right)$ 
\RETURN $\mathcal{C},\mathcal{U},\mathcal{R},\mathbf{I},\mathbf{J}$
\end{algorithmic}
\end{algorithm}

Theorem \ref{t4.2} suggests that we can over-sample indices to find the exact t-CUR decomposition of $\mathcal{A}$. To ensure the exact t-CUR decomposition and avoid missing important features, we use Algorithm \ref{a3} to sample $d_1=r_1\log(m)$ horizontal slices and $d_2=r_2\log(n)$ lateral slices.
 Then we enforce  the sizes of the sampled index sets to be  $|\mathbf{I}|=r_1$ and $|\mathbf{J}|=r_2$. The t-DEIM method can be useful for choosing such representative indices. This method,  an extension of the matrix DEIM method in \cite{TDEIM}, is presented in Algorithm \ref{a4}.
\begin{algorithm}
\caption{T-DEIM index selection}
\label{a4}
\begin{algorithmic}
\STATE $\mathbf{Input}:  \mathcal{U}\in \mathbb{R}^{m\times n\times p}$
\STATE $\mathbf{p}\left(1\right)=\mathop {\mathrm{arg}\max} \limits_{i\in \left\{ 1:m \right\}}\left\| \mathcal{U} \left(i,1,:\right) \right\| _{\fro}$
\FOR{$j=2:n$}
\STATE $\mathcal{R} =\mathcal{U}\left(:,j,:\right)-\mathcal{U}\left(:,1:j-1,:\right)\ast \mathcal{U}\left(\mathbf{p}\left(1:j-1\right),1:j-1,:\right)^{-1}\ast \mathcal{U} \left(\mathbf{p}\left(1:j-1\right),j,:\right)$
\STATE $\mathbf{p}\left(j\right)=\mathop {\mathrm{arg}\max} \limits_{i\in \left\{ 1:m \right\}}\left\| \mathcal{R} \left(i,j,:\right) \right\| _{\fro}$
\ENDFOR
\RETURN index set $\mathbf{p}$
\end{algorithmic}
\end{algorithm}

Combining Algorithms \ref{a3} and \ref{a4}, we provide the t-CUR-DEIM index selection method to construct the core tensor $\mathcal{A}_{\mathbf{IJ}:}$ in Algorithm \ref{a5}.

\begin{algorithm}
\caption{t-CUR-DEIM index selection}
\label{a5}
\begin{algorithmic}
\STATE $\mathbf{Input}$:  $\mathcal{A}\in \mathbb{R}^{m\times n\times p}$, extracting sizes $r_1,r_2$
\STATE $\left[ \sim,\mathcal{U} ,\sim,\mathbf{I},\mathbf{J} \right] =\textnormal{t-CUR}\left( \mathcal{A} ,r_1\log \left( m \right) ,r_2\log \left( n \right) \right) $
\STATE $\left[ \mathcal{W} ,\sim ,\mathcal{V} \right] =\textnormal{t-SVD}\left( \mathcal{U} \right) $
\STATE $\mathbf{p}=\textnormal{t-DEIM}\left( \mathcal{V} \right) $
\STATE $\mathbf{q}=\textnormal{t-DEIM}\left( \mathcal{W} \right) $
\STATE $\mathbf{I}=\mathbf{I}\left(\mathbf{p}\left( 1:r_1 \right)\right)$
\STATE $\mathbf{J}=\mathbf{J}\left(\mathbf{q}\left( 1:r_2 \right)\right)$
\RETURN $\mathbf{I},\mathbf{J}$, the core $\mathcal{A}\left(\mathbf{I},\mathbf{J},:\right)$
\end{algorithmic}
\end{algorithm}

After   the core tensor $\mathcal{A}_{\mathbf{I}\mathbf{J}:}$ is obtained,  the tensors $\mathcal{P}_1$ and $\mathcal{P}_2$ is solved by considering the following optimization problem:
\begin{align}\label{e4.6}
\left( \mathcal{P} _1,\mathcal{P} _2 \right) =\mathop {\mathrm{arg}\min} \limits_{\mathcal{Q} _1\in \mathbb{R} _{+}^{m\times r_1\times p}\atop	\mathcal{Q} _2\in \mathbb{R} _{+}^{r_2\times n\times p}}\left\| \mathcal{A} -\mathcal{Q} _1\ast\mathcal{A} _{\mathbf{IJ}:}\ast\mathcal{Q} _2 \right\| _{\fro}^{2}.\end{align}
When one of the factors, $\mathcal{Q}_1$ or $\mathcal{Q}_2$ is fixed, \eqref{e4.6} will be reduced to a tensor nonnegative least squares problem. And using Definition \ref{t2.2}, the Frobenius norm
$$\left\| \mathcal{A} -\mathcal{G}\ast \mathcal{H} \right\| _{\fro}$$ 
can be re-written in matrix form as
$$\left\| \mathsf{unfold}\left( \mathcal{A} \right) -\mathsf{bcirc}\left( \mathcal{G} \right) \mathsf{unfold}\left( \mathcal{H} \right) \right\| _{\fro},$$
allowing us to apply a coordinate descent method similar to that described in \cite{P1P2}, as shown in \Cref{a6}.

\begin{algorithm}
\caption{Computing $\mathcal{P}_1$ and $\mathcal{P}_2$}
\label{a6}
\begin{algorithmic}
\STATE $\mathbf{input}$\ $\mathcal{A}\in \mathbb{R}_+^{m\times n\times p}$, index sets $\mathbf{I}$, $\mathbf{J}$, the maximum number of iterations $maxiter$, stopping criterion $\delta$
\STATE solve $P_1=\mathop {\mathrm{arg}\min} \limits_{Q_1\in \mathbb{R} _{+}^{mp\times r_1p}}\left\| \mathsf{unfold}\left( \mathcal{A} \right) -Q_1\mathsf{unfold}\left( \mathcal{A} _{\mathbf{I}::} \right) \right\| _{\fro}^{2}$
\STATE $\mathcal{P}_1=\mathsf{fold}\left(P_1\mathsf{unfold}\left(\mathcal{I}\right)\right)$
\STATE solve $P_2=\mathop {\mathrm{arg}\min} \limits_{Q_2\in \mathbb{R} _{+}^{r_2p\times n}}\left\| \mathsf{unfold}\left( \mathcal{A} \right) -\mathsf{bcirc}\left( \mathcal{A} _{:\mathbf{J}:} \right) Q_2 \right\| _{\fro}^{2}$
\STATE $\mathcal{P}_2=\mathsf{fold}\left(P_2\right)$
\FOR {$i=1:maxiter$}
\STATE $\mathcal{P}^{\mathrm{old}}_1 = \mathcal{P}_1,\quad\mathcal{P}^{\mathrm{old}}_2 = \mathcal{P}_2$
\STATE solve $P_2=\mathop {\mathrm{arg}\min} \limits_{Q_2\in \mathbb{R} _{+}^{r_2p\times n}}\left\| \mathsf{unfold}\left( \mathcal{A} \right) -\mathsf{bcirc}\left( \mathcal{P}_1\ast\mathcal{A} _{\mathbf{IJ}:} \right) Q_2 \right\| _{\fro}^{2}$
\STATE $\mathcal{P}_2=\mathsf{fold}\left(P_2\right)$
\STATE solve $P_1=\mathop {\mathrm{arg}\min} \limits_{Q_1\in \mathbb{R} _{+}^{mp\times r_1p}}\left\| \mathsf{unfold}\left( \mathcal{A} \right) -Q_1\mathsf{unfold}\left( \mathcal{A} _{\mathbf{IJ}:}\ast\mathcal{P}_2 \right) \right\| _{\fro}^{2}$
\STATE $\mathcal{P}_1=\mathsf{fold}\left(P_1\mathsf{unfold}\left(\mathcal{I}\right)\right)$
\IF {$\left\|\mathcal{P}^{\mathrm{old}}_1-\mathcal{P}_1\right\|_{\fro}+\left\|\mathcal{P}^{\mathrm{old}}_2-\mathcal{P}_2\right\|_{\fro}\leqslant\delta$}
\STATE break
\ENDIF
\ENDFOR
\RETURN $\mathcal{P}_1$, $\mathcal{P}_2$
\end{algorithmic}
\end{algorithm}

\section{Numerical experiments}\label{sec:simulation}
In this section, we present the performance evaluations of our proposed index selection algorithms, CoS-NTF and t-CUR-DEIM (referred to as t-CUR), on both synthetic and real-world datasets. Specifically, we have selected some facial databases as real-world datasets, which can be properly arranged in the tensor format.

For the CoS-NTF index selection, we employ two different stopping criteria for the synthetic datasets and the facial datasets, respectively. For the synthetic dataset, we set the threshold $\delta= 10^{-6}$, while for the facial datasets, we set $\delta=10^{-2}$. For all simulations, the maximum iterative number is bounded by $50$.  
In the SNMF-FGM step of \cref{a2}, the output $\mathbf{I,J}$ are set to be the indices of the $r_1$ largest diagonal entries of $X$ and the $r_2$ largest diagonal entries of $Y$ respectively. The regularization parameters in \eqref{e3.10}-\eqref{e3.11} are fixed at 0.25 after experimenting with values ranging from $10^{-3}$ to $1$.

In the t-CUR index selection process, we evaluate the performance of three different random sampling distributions \eqref{e4.3}-\eqref{e4.5} for synthetic datasets. It's important to note that while leverage score sampling involves the computation of t-SVD, and slice-size sampling requires calculating the tensor Frobenius norm, generating both random distributions is relatively slower compared to uniform sampling. Considering these computational demands, uniform sampling is chosen as the preferred method for facial datasets to prioritize computational efficiency.

All the tests are executed from MATLAB R2022a on a laptop. The system specifications include an AMD Ryzen 7 4800H processor, featuring a 2.9 GHz clock speed and 8 cores, coupled with 16 GB of DDR4 RAM operating at 3200MHz. For graphical computations and processing, we utilize an NVIDIA GeForce GTX 1650 Ti GPU. This hardware runs on a 64-bit operating system, which is based on an x64 processor architecture. 

\subsection{Synthetic data sets}
With \eqref{e3.3}, we generate the noisy co-($10,3$)-separable tensor
$$\mathcal{A}=\varPi _1\ast\max \left\{ 0,\underset{\mathcal{A}^{\mathrm{scale}}}{\underbrace{\mathcal{D}_r\ast\left( \begin{matrix}
	\mathcal{S}&		\mathcal{S}\ast\mathcal{H}\\
	\mathcal{M}\ast\mathcal{S}&		\mathcal{M}\ast\mathcal{S}\ast\mathcal{H}\\
\end{matrix} \right)\ast \mathcal{D}_c}}+\mathcal{N} \right\} \ast\varPi _2\in\mathbb{R}_+^{100\times100\times10},$$
where entries of tensors $\mathcal{S}\in\mathbb{R}_+^{10\times3\times10}$, $\mathcal{M}\in\mathbb{R}_+^{90\times10\times10}$, and $\mathcal{H}\in\mathbb{R}_+^{3\times97\times10}$ are independently sampled from a  uniform distribution between  $[0,1]$. The scaling tensors $\mathcal{D}_r$ and $\mathcal{D}_c$ adjust $\mathcal{A}^{\mathrm{scale}}$ such that  the sum of each slice $\mathcal{A}^{\mathrm{scale}}_{i::}$ or $\mathcal{A}^{\mathrm{scale}}_{:j:}$ sums to  $100$. Then $\mathcal{A}^{\mathrm{scale}}$ is the noiseless scaled co-($10,3$)-separable tensor. 

The noise tensor $\mathcal{N}\in\mathbb{R}^{100\times100\times10}$ is generated from  a standard normal distribution. The noise magnitude is normalized so that $\|\mathcal{N}\|_{\fro}=\epsilon\|\mathcal{A}^{\mathrm{scale}}\|_{\fro}$ with  noise levels $\epsilon$ varying from $10^{-7}$ to $10^{-1}$. $\varPi_1$ and $\varPi_2$ are permutation tensors with $\left(\varPi _1\right)_{::1}$ and $\left(\varPi _2\right)_{::1}$ generated randomly and with the remaining entries being  zeros. 

For each noise level, we generated 10 different tensor  $\mathcal{A}$  and evaluated their average relative errors
$$\left\| \mathcal{A} -\widetilde{\mathcal{A}} \right\| _{\fro}/\left\| \mathcal{A} \right\| _{\fro},\quad \widetilde{\mathcal{A}}=\widetilde{\mathcal{P}}_1\ast \mathcal{A} _{\mathbf{\widetilde{I}\widetilde{J}}:}\ast \widetilde{\mathcal{P}}_2$$
where $\widetilde{\mathbf{I}}$, $\widetilde{\mathbf{J}}$ are selected by  \Cref{a2,a5}. and $\widetilde{\mathcal{P}}_1$, $\widetilde{\mathcal{P}}_2$ are computed by \Cref{a6} with $\textnormal{maxiter}=100$ and $\delta=10^{-6}$. The results are shown in \Cref{f3}.

\begin{figure}[htbp]
\begin{center}
\includegraphics[width=15cm]{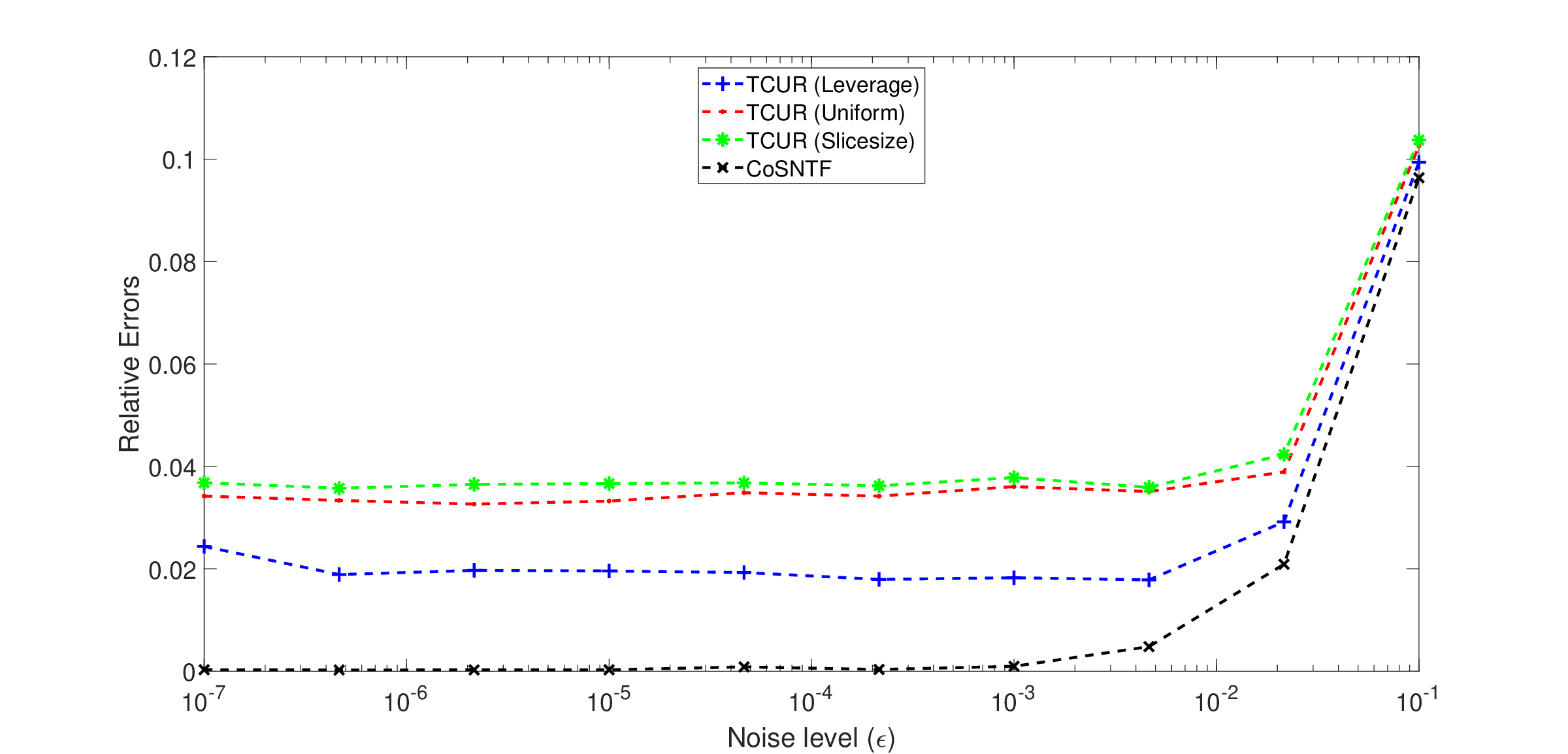}
\caption{Average relative error on the fully randomly generated coseparable tensors.}
\label{f3}
\end{center}
\end{figure}

From  \Cref{f3}, one can see our CoS-NTF index sampling achieves low relative errors, which proves its feasibility and shows its ability to accurately select indices for generating the coseparable core tensor. Meanwhile, the t-CUR-DEIM index selection methods show larger relative errors across all three sampling distributions in repeated experiments. Since \cref{a3} randomly choose $r_1\log m$ and $r_2 \log n$ indices and we use \cref{a4} to further choose $r_1$ and $r_2$ representative indices based on it, the resulting core of \cref{a5} sometimes may not meet the assumption of Theorem \ref{t3.5} (ii), then can't ensure an exact coseparability. Therefore, given a strictly coseparable tensor, the t-CUR method falls short in accurately reconstructing the coseparable tensors compared to CoS-NTF index sampling. However, it still can choose significant indices in most cases, as shown by numerical experiment results. 

Then we delve into a detailed analysis of the t-CUR index sampling. Since the tensor $\mathcal{A}$ is generated from the uniform distribution, each slice will have nearly the same sampling probability for the slice-size sampling, resulting in a similar result with the uniform sampling. We can also see that the leverage sampling performs slightly better because it gives more sampling weights to rows or columns with more information than the other two sampling distributions. 

\subsection{Facial data sets}
In this section, we apply our proposed CoS-NTF and t-CUR index selection algorithms on facial data sets to test their clustering abilities. The data sets we used are briefly described as follows.
\begin{itemize}
\item \textbf{The Extended Yale Face Database B (YaleB)} \cite{YaleB} contains 2414 grayscale images from 38 distinct individuals captured under 64 different lighting conditions. Each image has an original size of $192\times168$ pixels, which we subsequently resize to $24\times21$ pixels. We also selected 20 face images for each individual at random to obtain 760 images. Since the same individuals can be regarded as the same cluster and $r_2$ represents the clustering number, we set $r_2=38$. $r_1=24$ is set to be the row number of the resized images. 
\item \textbf{The ORL Database of Faces (ORL)} \cite{ORL} contains 400 grayscale images from 40 distinct subjects, where each subject has ten different images taken at different lighting, facial expressions, and facial details. The size of each image is $112\times92$ pixels and we resize them to $23\times19$ pixels. We set $r_1=23$, $r_2=40$.
\item \textbf{The Frey Face Dataset (Frey)} is collected by Brendan Frey, which contains 1965 grayscale images of Brendan Frey's face, taken from sequential frames of a small video. The size of each image is $20\times28$, so we set $r_1=r_2=20$.
\item \textbf{The MIT-CBCL Face Database \#1 (CBCL)} \cite{CBCL} consists of a training set of 2429 grayscale face images and a test set of 472 face images of size $19\times19$ pixels. We take the training data set and set $r_1=19$ and $r_2=49$ as in \cite{CBCL2}.
\item \textbf{The UMIST Face Database (UMist)} \cite{UMist} consists of 575 grayscale face images of 20 different people. Each image has $112\times92$ size and we resize them into $28\times23$ pixels. We set $r_1=r_2=20$.
\end{itemize}

For each dataset, we arrange all the images into a tensor $\mathcal{A}\in\mathbb{R}_+^{m\times n\times p}$, where $m$ and $p$ stand for the size of the image and $n$ is the number of images. For example, $\mathcal{A}^{\mathrm{YaleB}}\in\mathbb{R}_+^{24\times 760\times 21}$, $\mathcal{A}^{\mathrm{ORL}}\in\mathbb{R}_+^{23\times 400\times 19}$.

Note that some matrix index selection methods can also be used to cluster the facial data sets, including the Coseparable NMF method, where they vectorize each image and arrange them into a matrix, like ${A}^{\mathrm{YaleB}}\in\mathbb{R}_+^{\left(24\times 21\right)\times 760}$ and ${A}^{\mathrm{ORL}}\in\mathbb{R}_+^{\left(23\times 19\right)\times 400}$ for the data sets above. We also compare our methods with these state-of-the-art index selection methods. We first summarize them as follows.
\begin{itemize}
\item \textbf{The Successive Projections Algorithm (SPA)} \cite{SPA} is a forward method to select the column of the matrix $A$ with the largest $l_2$ norm and then calculate the projections of all columns on the orthogonal complement of the selected column at each step. We will apply SPA on $A$ and $A^\top$ to select $r_2$ columns and $r_1$ rows of $A$ respectively. We call it \textbf{SPA+} method.
\item \textbf{The Coseparable NMF Fast Gradient Method (CoS-NMF)} is an alternating fast gradient method to solve the coseparable NMF problem proposed in \cite{pan2021co} and determine the rows and columns of the coseparable core. We follow the same settings as \cite{pan2021co}, i.e., $\delta=10^{-2}$, $maxiter=1000$ and using the simplest postprocessing to select $r_1$ rows and $r_2$ columns in the SNMF-FGM step. The regularization parameter $\lambda$ is fixed at 0.25 after experimenting with values ranging from $10^{-3}$ to $1$.
\item \textbf{The Matrix CUR sampling (CUR)} is a randomized sampling method to select rows and columns for the CUR decomposition factor, the submatrix $U$. Its sampling stability and complexity have been proved in \cite{CUR}. Similar to our t-CUR index sampling, we will uniformly sample $r_1\log(m)$ rows and $r_2\log(n)$ columns from the $m\times n$ original data matrix, and then use matrix DEIM \cite{DEIM} to enforce the sampled submatrix to be $r_1\times r_2$.
\end{itemize}

As our CoS-NTF method can directly handle the facial data as a whole, it needs additional $\mathsf{unfold}$ operations in each step. Moreover, the unfolded matrix may have a highly imbalanced row-column ratio, resulting in a significant increase in runtime. However, we have proved the relationship between the t-CUR decomposition and coseparable NTF, as well as the t-CUR sampling theory. Motivated by these findings, we adopt a hybrid approach: first pre-selecting indices randomly, such as $r_1\log(m)$ horizontal slices and $r_2\log(n)$ lateral slices, and then applying the CoS-NTF index selection on this pre-selected subtensor. This strategy effectively reduces runtime.

After selecting the row and column indices $\tilde{\mathbf{I}},\tilde{\mathbf{J}}$, we adopt the \cite[Algorithm 3]{pan2021co} to compute their nonnegative factors $\tilde{P}_1,\tilde{P}_2$ and reconstruct the approximate solutions $\tilde{A}=\tilde{P}_1A_{\tilde{\mathbf{I}}\tilde{\mathbf{J}}}\tilde{P}_2$ for the three matrix methods. For our t-CUR and CoS-NTF methods, we compute the tensors $\widetilde{\mathcal{P}}_1$ and $\widetilde{\mathcal{P}}_2$ by the Algorithm \ref{a6} and reach the approximate solutions $\widetilde{\mathcal{A}}=\widetilde{\mathcal{P}}_1\ast \mathcal{A} _{\mathbf{\widetilde{I}\widetilde{J}}:}\ast \widetilde{\mathcal{P}}_2$ as section 5.1. In Table \ref{tb1}, we show the relative approximations (in percent)
$$1-\frac{\left\| {A} -\tilde{A} \right\| _{\fro}}{\left\| {A} \right\| _{\fro}},\quad1-\frac{\left\| \mathcal{A} -\widetilde{\mathcal{A}} \right\| _{\fro}}{\left\| \mathcal{A} \right\| _{\fro}}.$$
\begin{table}[htbp]\centering
\caption{Relative approximations for the facial datasets.}
\label{tb1}
\begin{tabular}{c||c|c|c|c|c|c}
\hline\hline
Database	&$\left(r_1,r_2\right)$	&CoS-NTF	&t-CUR	&SPA+	&CoS-NMF	&CUR\\
\hline
YaleB	&$\left(24,38\right)$	&\rev{77.20}	&\rev{\textbf{78.82}}	&\rev{76.96}	&\rev{73.36}	&\rev{76.13}\\
ORL	&$\left(23,40\right)$	&\rev{84.53}	&\rev{\textbf{85.12}}	&\rev{81.53}	&\rev{80.28}	&\rev{81.90}\\
Frey	&$\left(20,20\right)$	&\rev{88.92}	&\rev{\textbf{89.17}}	&\rev{88.56}	&\rev{86.40}	&\rev{88.51}\\
CBCL	&$\left(19,49\right)$	&\rev{\textbf{82.45}}	&\rev{82.33}	&\rev{80.46}	&\rev{79.55}	&\rev{79.86}\\
UMist	&$\left(20,20\right)$	&\rev{75.22}	&\rev{\textbf{78.15}}	&\rev{75.02}	&\rev{73.21}	&\rev{74.57}\\
\hline
\end{tabular}
\end{table}

Table \ref{tb1} shows that our CoS-NTF and t-CUR methods have higher relative approximation on all 5 databases, which indicates they perform better in clustering than the matrix methods. It also demonstrates that handling the high dimensional data directly and keeping their high-order structure can better retain the original information.

 {Besides, since the real datasets usually have full rank and are only approximately coseparable, we only need to find their approximate coseparable cores. In this case, the t-CUR method performs as well as the CoS-NTF method (see~\cref{a2}), even better.} And since t-CUR sampling is a randomized method and the CoS-NTF method needs to run the fast gradient method iteratively, the t-CUR method runs faster, which shows the potential of our proposed t-CUR sampling theory.

Furthermore, the matrix CUR sampling method also performs well among the three matrix methods. It is an efficient way to sample row/column indices for the matrix.

\section{Conclusion}\label{sec:conclusion}
In this paper, we extend the concept of coseparability to tensors under the t-product framework and propose coseparable NTF. For this factorization, we investigate its characterizations and its relationship with t-CUR decomposition. Using the SNMF-FGM algorithm, we also introduce an algorithm to select the coseparable indices alternately. Additionally, based on t-CUR decomposition, we propose the t-CUR sampling theory and combine it with the t-DEIM method to establish another random algorithm, the t-CUR-DEIM algorithm, for selecting coseparable indices. Then for the selected coseparable cores, we need to solve alternating tensor least squares problems to obtain the nonnegative factors. Therefore, we use the definition of the t-product to unfold the tensor and solve the corresponding matrix least squares problems.

For the proposed t-CUR and CoS-NTF methods, we first test them with synthetic coseparable tensors. The result shows that the CoS-NTF method effectively selects coseparable indices, and the t-CUR method has relatively larger errors due to its randomness, where the leverage score sampling can assign more sampling weights to slices with more information, and outperforms uniform sampling and slice-size sampling. Furthermore, on the real facial data sets, we compare their clustering abilities with three matrix sampling methods. The results suggest that our t-CUR and CoS-NTF methods require more runtime since they should unfold tensors and the unfolded tensors have an unbalanced row-column ratio in each iterative computation. Due to our methods processing each image as a whole rather than vectorizing it, they can retain more original information compared to the three matrix methods, which results in higher relative approximations. It shows the potential of the proposed tensor methods for handling high-dimensional data.

\section*{Acknowledgements}

Some of the work for this article was conducted while J.C. was a visiting student at Michigan State University. The authors would like to express their gratitude to the handling editor, and  two anonymous reviewers for their insightful suggestions regarding our manuscript.

\bibliographystyle{siamplain}
\bibliography{ref}

\end{document}